\documentclass[letterpaper]{article}
\usepackage{aaai}
\usepackage{times}
\usepackage{helvet}
\usepackage{courier}
\usepackage[defblank]{paralist}
\usepackage{xspace}
\usepackage{latexsym,amssymb,amsmath,amsthm}
\usepackage{graphicx}
\usepackage{afterpage}
\usepackage{algorithm}
\usepackage[noend]{algpseudocode}
\usepackage{todonotes}
\usepackage{mdframed}
\usepackage{soul}
\makeatletter
\newcommand{\nobrackettag}[0]{\def\tagform@##1{\maketag@@@{##1}}}
\makeatother
\newcommand{\qedboxfull}{\vrule height 5pt width 5pt depth 0pt}
\newcommand{\qedfull}{\hfill{\qedboxfull}}

\newcommand{\A}{\mathcal{A}} 
\newcommand{\C}{\mathcal{C}} \newcommand{\D}{\mathcal{D}}
 \newcommand{\F}{\mathcal{F}}
\newcommand{\G}{\mathcal{G}} \renewcommand{\H}{\mathcal{H}}
 
 \renewcommand{\L}{\mathcal{L}}
\newcommand{\M}{\mathcal{M}} \newcommand{\N}{\mathcal{N}}
 \renewcommand{\P}{\mathcal{P}}
 \newcommand{\R}{\mathcal{R}}
\renewcommand{\S}{\mathcal{S}} \newcommand{\T}{\mathcal{T}}
\newcommand{\U}{\mathcal{U}} 
 \newcommand{\X}{\mathcal{X}}

\newcommand{\tup}[1]{\langle #1\rangle}   % tuple
\newcommand{\set}[1]{\{#1\}}                      % set

\newcommand{\inst}{\mathsf{inst}}

\newcommand{\outtype}[1]{\textsc{out-type}(#1)}
\newcommand{\outtypec}[2]{\textsc{out-type}_{#2}(#1)}

\newcommand{\ex}[1]{\mathit{#1}}
\newcommand{\true}{\mathsf{true}}
\newcommand{\false}{\mathsf{false}}

\newcommand{\BOX}[1]{ [\!{-}\!] #1}
\newcommand{\DIAM}[1]{\langle\!{-}\!\rangle #1}

\newcommand{\adom}[1]{\textsc{adom}(#1)}

\newcommand{\ans}[2][]{\mathit{ans}_{#1}(#2)}

\newcommand{\map}[2]{#1 \rightsquigarrow #2}

\newcommand{\calls}[1]{\textsc{calls}(#1)}

\newcommand{\cspec}{\ensuremath{\C}\xspace}

\newcommand{\db}{\ensuremath{db}\xspace}
\newcommand{\instspec}{I}

\newcommand{\mulpd}{\ensuremath{\mu\L^{\scriptsize @}_p}\xspace}

\newcommand{\inadom}{\ensuremath{\textsc{live}}}
\newcommand{\limp}{\rightarrow}

\newcommand{\myname}{\mathit{MyName}}

\newcommand{\ena}{\textbf{enables}}
\newcommand{\froma}{\textbf{from}}
\newcommand{\toa}{\textbf{to}}
\newcommand{\onev}{\textbf{on}}

\newcommand{\ifc}{\textbf{if}}
\newcommand{\thendo}{\textbf{then}}
\newcommand{\ev}[1]{\textsc{#1}}
\newcommand{\add}{\textnormal{\textbf{add}}}
\newcommand{\del}{\textnormal{\textbf{del}}}

\newcommand{\agent}{\ensuremath{\ex{Agent}}\xspace}
\newcommand{\spec}{\ensuremath{\ex{Spec}}\xspace}
\newcommand{\hasSpec}{\ensuremath{\ex{hasSpec}}\xspace}

\newcommand{\afacet}{AF}
\newcommand{\sfacet}{BF}

\newcommand{\str}[1]{``\cv{#1}"}

\newcommand{\agset}{\G}
\newcommand{\evset}{\M}
\newcommand{\dtset}{\T}
\newcommand{\dfset}{\F}
\newcommand{\dt}{T}
\newcommand{\df}{F}
\newcommand{\rel}[1]{\R_{#1}}
\newcommand{\drel}[1]{\R_{#1}}

\newcommand{\scname}[1]{\ensuremath{\mathbf{#1}}\xspace} % service call name (lowercase)
\newcommand{\pname}[1]{\ensuremath{\mathsf{#1}}\xspace} % param name (lowercase)
\newcommand{\actname}[1]{\ensuremath{\textsc{#1}}\xspace} % param name (lowercase)

\newcommand{\buff}{\mathit{MBuffer}}
\newcommand{\newm}{\mathit{NewM}}
\newcommand{\oldm}{\mathit{OldM}}

\newcommand{\tdom}[1]{\Delta_{#1}}
\newcommand{\ddom}[1]{\Delta_{#1}}
\newcommand{\cdom}[1]{\Delta_{0,#1}}

\newcommand{\succp}{\mathsf{succ}}
\newcommand{\idb}{D_0}

\newcommand{\schema}{\D}

\newcommand{\crule}{\C}
\newcommand{\urule}{\U}
\newcommand{\uact}{\A}
\newcommand{\constr}{\Gamma}

\newcommand{\name}{\cname{n}}

\newcommand{\scset}{\S}

\newcommand{\cname}[1]{\mathtt{#1}}
\newcommand{\typed}[1]{\overline{#1}}
\newcommand{\type}[3]{\textsc{type}_{#3}(#1[#2])}
\newcommand{\typer}[2]{\textsc{type}_{#2}(#1)}
\newcommand{\rmas}{\tup{\dtset, \dfset, \cdom{\dfset},\scset, \evset,
    \agset,\instspec}}
\newcommand{\sys}{\X}

\newcommand{\ts}{\Upsilon}
\newcommand{\pts}{\Lambda}
\newcommand{\fts}{\Theta}

\newcommand{\trans}{{\rightarrow}}

\newcommand{\subst}{\theta}
\newcommand{\state}{s}
\newcommand{\cv}[1]{\mathtt{#1}}

\newcommand{\flatten}{\textsc{flatten}}

\newcommand{\cm}{\mathfrak{C}\xspace}
\newcommand{\instr}[1]{\mathsf{#1}}
\newcommand{\halt}{\instr{HALT}}
\newcommand{\inc}[2]{\instr{INC}(#1,#2)}
\newcommand{\cdec}[3]{\instr{CDEC}(#1,#2,#3)}
\newcommand{\goto}[1]{\instr{GOTO}\ #1}

\newcommand{\cell}[2]{[#1]_{#2}}

\newcommand{\lessthan}[1]{\ex{lessThan}_{#1}}

\usepackage{url}
\newtheorem{counter}{Counter}[section]
\newtheorem{lemma}[counter]{Lemma}

\newtheorem{theorem}[counter]{Theorem}
\newtheorem{exampleAux}{Example}[section]
\newenvironment{example}{\begin{exampleAux}\upshape}{\qedfull\end{exampleAux}}

\newmdenv[
 linecolor=yellow,backgroundcolor=yellow,
 outerlinewidth=0,leftmargin=0cm,rightmargin=0cm,innerleftmargin=0cm,innertopmargin=0cm,innerbottommargin=0cm,innerrightmargin=0cm,
skipabove=0pt,skipbelow=0cm
]
{infobox}

\sethlcolor{yellow}

\begin{document}
% The file aaai.sty is the style file for AAAI Press
% proceedings, working notes, and technical reports.
%
\title{Verification of Relational Multiagent Systems with Data Types\\
(Extended Version)}
\author{Diego Calvanese \qquad Marco Montali\\
 Free University of Bozen-Bolzano\\
 Piazza Domenicani 3, 39100 Bolzano, Italy\\
 \url{{calvanese,montali}@inf.unibz.it}
 \And
 Giorgio Delzanno\\
 University of Genova\\
 Via Dodecaneso 35, 16146 Genova, Italy\\
 \url{giorgio.delzanno@unige.it}
}
\maketitle

\begin{abstract}
  We study the extension of relational multiagent systems (RMASs), where agents
  manipulate full-fledged relational databases, with data types and facets
  equipped with domain-specific, rigid relations (such as total orders).
  Specifically, we focus on design-time verification of RMASs against rich
  first-order temporal properties expressed in a variant of first-order
  $\mu$-calculus with quantification across states.  We build on previous
  decidability results under the ``state-bounded'' assumption, i.e., in each
  single state only a bounded number of data objects is stored in the agent
  databases, while unboundedly many can be encountered over time.  We recast
  this condition, showing decidability in presence of dense, linear orders, and
  facets defined on top of them.  Our approach is based on the construction of
  a finite-state, sound and complete abstraction of the original system, in
  which dense linear orders are reformulated as non-rigid relations working on
  the active domain of the system only.  We also show undecidability when
  including a data type equipped with the successor relation.
\end{abstract}

\section{Introduction}

We study \emph{relational multiagent systems} (RMASs), taking inspiration from
the recently defined framework of data-aware commitment-based multiagent
systems (DACMASs) \cite{ChSi13,MoCD14}.  Broadly speaking, an RMAS is
constituted by agents that maintain data in an internal full-fledged relational
database, and apply proactive and reactive rules to update their own data, and
exchange messages with other agents. Messages have an associated payload, which
is used to move data from one agent to another.  Notably, when updating their
internal database, agents may also inject fresh data into the system, by
invoking external services. This abstraction serves as a metaphor for any kind
of interaction with the external world, such as invocation of web services, or
interaction with humans.

From the data perspective, previous research has mainly focused on a single,
countably infinite data domain, whose elements can only be compared for
equality and inequality.  This assumption is highly restrictive, since data
types used in applications are typically equipped with domain-specific, rigid
relations (such as total orders), and might be specialized through the use of
\emph{facets} \cite{ISOGPD07,SaCa12}.

The focus of this work is on design-time verification of RMASs against rich
first-order temporal properties, allowing for quantification across states.  By
considering only a countably infinite domain with equality, it has been shown
in \cite{BeLP12,BCDDM13,MoCD14} that decidability of verification holds for
variants of first-order temporal logics under the assumption that the system is
``state-bounded'', i.e., unboundedly many data objects can be encountered over
time, provided that in each single state only a bounded number of them is
stored in the agent databases \cite{BCDM14}. We recast this condition by
considering different options for the data types. Specifically, by exploiting
an encoding of two-counter machines, we show that decidability of verification
even of propositional reachability properties is lost when one of the data
types is equipped with the successor relation.
Our main technical result is showing decidability for a variant of first-order
$\mu$-calculus in presence of dense, linear orders, and facets defined on top
of them. In this case, we provide an explicit technique to construct a
finite-state, sound and complete abstraction of the original system, in which
dense linear orders are reformulated as non-rigid relations working on the
active domain of the system only.  Notably, this allows us to model and verify
state-bounded RMASs that include coordination mechanisms such as ticket-based
mutual exclusion protocols.

%Full proofs and examples are available in the appendix.

\section{Relational Multiagent Systems}

RMASs are data-aware multiagent systems constituted by agents that
exchange and update data.
%
% Messages have an associated relational payload, which is used to exchange
% data. Furthermore, when updating their internal database, agents may also
% inject fresh data into the system, by invoking external services. The
% abstraction of external service serves as a metaphor for any kind of
% interaction with the external world, such as invocation of web services, or
% interaction with human users.
%
Beside generic agents, an RMAS is equipped with a
so-called \emph{institutional agent}, which exists from the initial system
state, and can be contacted by the other agents as a sort of ``white-page'' agent,
i.e., to:
\begin{inparaenum}[\it (i)]
\item get information about the system as a whole;
\item obtain names of other agents so as to establish an interaction with them;
  and
\item create and remove agents.
\end{inparaenum}
% In \cite{MoCD14}, the institutional agent is also responsible of
% monitoring the agent interactions and infer their impact on
% first-order, commitment-based contracts. Although we do not specifically delve
% into social commitment here, all the provided results can be transferred to the
% setting where such contracts are present.

At a surface level, RMASs and DACMASs share many aspects. There are however two
key differences in the way they model data. On the one hand, while DACMASs
consider only a single, abstract data domain equipped with equality only, in
RMASs data are typed and enriched with domain-specific relations. This deeply
impacts the modeling power of the system (see Section~\ref{sec:example}). On
the other hand, while agents in DACMASs operate with incomplete knowledge about
the data, and use a description logic ontology as a semantic interface for
queries, RMASs employ standard relational technology for storage and querying
services. This is done to simplify the treatment and isolate the core issues
that arise when incorporating data types and facets, but we believe our results
can be transferred to DACMASs as well.

%Formally, we have:

% \subsection{The RMAS Framework}
% We formally introduce RMASs, discussing each component.
%\begin{definition}
An \emph{RMAS} $\sys$ is a tuple $\rmas$, where:
\begin{inparaenum}[(1)]
\item $\dtset$ is a finite set of \emph{data types};
\item $\dfset$ is a finite set of \emph{facets} over $\dtset$;
\item $\cdom{\dfset}$ is the initial data domain of $\sys$;
\item $\scset$ is a finite set of \emph{$\dfset$-typed service calls};
\item $\evset$ is a finite set of \emph{$\dfset$-typed relations} denoting
  messages with payload;
\item $\agset$ is a finite set of \emph{$\dfset$-typed agent specifications};
  and
\item $\instspec$ is the $\dfset$-typed specification of the
  \emph{institutional agent}.
\end{inparaenum}
%\end{definition}

\subsection{Data Types and Their Facets}
Data types and facets provide the backbone for modeling real-world objects manipulated by the RMAS agents.
%%
%\begin{definition}
A \emph{data type} $\dt$ is a pair $\tup{\tdom{\dt},\rel{\dt}}$, where
$\tdom{\dt}$ is an infinite set\footnote{Being infinite does not lead
  to a loss of generality, thanks to the notion of facet defined
  below.}, and $\rel{\dt}$ is a set of relation schemas. Each relation
schema $R/n \in \rel{\dt}$ with name $R$ and arity $n$ is associated with an $n$-ary predicate $R^\dt \subseteq \tdom{\dt}^n$.
%\end{definition}
Given a set $\dtset$ of data types, we denote by $\drel{\dtset}$
all domain-specific relations mentioned in $\dtset$.
% \[
% \drel{\dtset} = \set{R/n \mid R/n \in \rel{\dt} \text{ for some}
%   \tup{\tdom{\dt},\rel{\dt}} \in \dtset}
% \]
Similarly, $\ddom{\dtset}$ groups all the (pairwise disjoint) data domains of
the data types in $\dtset$.
% $
% \ddom{\dtset} = \biguplus_{\tup{\tdom{\dt},\rel{\dt}} \in \dtset} \tdom{\dt}
% $.
% data types have non-overlapping domains.
The interaction between data types is orthogonal to our work and is left for
the future.
% the results here presented, and is left as a future work.

\begin{example}
We consider the following, well-known data domains, whose relations retain the usual meaning:
\begin{compactitem}
% \item Strings, names, identifiers, all modeled by data types of the
%   form $\tup{\mathbb{S},\set{=}}$, where $\mathbb{S}$ is a countably
%   infinite set.
\item Dense total orders such as $\tup{\mathbb{Q},\set{<,=}}$ and
  $\tup{\mathbb{R},\set{<,=}}$.
\item Total orders with successor, like:
  % {$\tup{\mathbb{Q},\set{<,=,\succp}}$},
%  {$\tup{\mathbb{R},\set{<,=,\succp}}$}, and
{$\tup{\mathbb{Z},\set{<,=,\succp}}$}.
\end{compactitem}
\vspace*{-.4cm}
\end{example}

\noindent We assume that every RMAS has two special datatypes:
\begin{inparaenum}[\it (i)]
\item $\tup{\mathbb{A},\set{=}}$ for \emph{agent names} that, as in mobile
  calculi, behave as pure names \cite{Need89,MoPi05} and can only be tested for
  (in)equality.
\item $ \tup{\mathbb{B},\set{=}}$ for \emph{agent
    specification names} (see Section~\ref{sec:spec}).
\end{inparaenum}

Facets are introduced to restrict data types.
%%
%\begin{definition}
A \emph{facet} $\df$ is a pair $\tup{\dt,\varphi(x)}$ where $\dt
=\tup{\tdom{\dt},\rel{\dt}}$ is a data type, and $\varphi(x)$ is a monadic
\emph{facet formula} built as:
\[
\varphi(x) := \true \mid P(\vec{v}) \mid \neg \varphi(x) \mid \varphi_1(x) \lor \varphi_2(x)
\]
where $P(\vec{v})$ is a relation whose schema belongs to $\rel{\dt}$,
and whose terms $\vec{v}$ are either variable $x$ or data objects in $\tdom{\dt}$. We use the standard abbreviations $\false$ and $\varphi_1(x) \land \varphi_2(x)$.
%\end{definition}
Given a set $\dfset$ of facets, we
use $\drel{\dfset}$ and $\ddom{\dfset}$ as a shortcut for
$\drel{\dtset}$ and $\ddom{\dtset}$ respectively, where $\dtset$ is
the set of data types on which facets in $\dfset$ are
defined.

%\begin{definition}
Given a \emph{facet} $\df = \tup{\dt,\varphi(x)}$ with
$\dt=\tup{\tdom{\dt},\rel{\dt}}$, a data object $\cname{d}$ \emph{belongs to}
$\df$ if:
\begin{inparaenum}[\it (i)]
\item $\cname{d} \in \tdom{\dt}$;
\item $\varphi(x)$ holds in $\df$ under substitution $[x/\cname{d}]$, written
  $\df,[x/\cname{d}] \models \varphi(x)$. In turn, given substitution $\sigma =
  [x/\cname{d}]$, relation $\df,\sigma \models \varphi(x)$ is inductively
  defined as follows:

\vspace*{-.5cm}
\[
\small
\begin{array}{@{}l@{~}c@{~~}l}
  \df,\sigma \models \true\\
  \df,\sigma \models R(\vec{v})\sigma & \text{if} & R(\vec{v})\sigma \text{ is true in } \dt\\
  \df,\sigma \models \neg \varphi(x) & \text{if} & \df,\sigma \not\models  \varphi(x) \\
  \df,\sigma \models \varphi_1(x) \land \varphi_2(x) & \text{if} &  \df,\sigma \models \varphi_1(x) \text{ and }  \df,\sigma \models \varphi_2(x)
\end{array}
\]
\end{inparaenum}
%\end{definition}

Notice that a \emph{base facet} that simply ranges over
all data objects of a data type can be encoded with $\true$ as its facet
formula. In particular, we use $\afacet =
\tup{\tup{\mathbb{A},\set{=}},\true}$ and  $\sfacet =
\tup{\tup{\mathbb{B},\set{=}},\true}$ to refer to two base facets for agent
and specification names respectively.

 \begin{example}
\label{ex:bool}
 An Enumeration $\cname{s_1},\ldots,\cname{s_n}$ over string values can be modeled as facet
 $\tup{\tup{\mathbb{S},\set{=}},\begin{array}{@{}l@{}}\bigvee_{i \in \set{1,\ldots,n}} x = \cname{s_i}}\end{array}$.
 This also accounts for the type of boolean, which can be captured by
 $\mathit{Bool} =  \tup{\tup{\mathbb{S},\set{=}},x=``\cv{t}" \lor x =
   ``\cv{f}"}$.
\end{example}

\begin{example}
%Let person ages be modeled with real numbers. Facet
$\tup{\tup{\mathbb{R},\set{>,=}},(x > \cname{0} \land \cname{18} > x)
  \lor x > \cname{65}} $
denotes ages of junior or senior people.
\end{example}

Facets are used as relation types.
%%
%\begin{definition}
Given a set $\dfset$ of facets, an \emph{$\dfset$-typed relation schema}
$\typed{R}$ is a pair $\tup{R/n,\dfset_R}$, where $R/n$ is a relation
schema with name $R$ and arity $n$, and $\dfset_R$ is an $n$-tuple
$\tup{\df_1,\dots,\df_n}$ of facets in $\dfset$.

%We extend the notion of component type to the case of domain-specific
%relations: every component of a domain-specific
%relation has as type that of the relation itself.
%%\begin{definition}
An \emph{$\dfset$-typed database schema} $\schema$ is a finite set of $\dfset$-typed relation schemas, such that no two typed relations in
$\schema$ share the same name.
%\end{definition}
%

In the following, we denote the $i$-th component of $R$ as
 ${R}[i]$, and write $\type{R}{i}{\schema}$ to indicate the type
 associated by $\schema$ to $R[i]$. We also denote the tuple of types
 associated by $\schema$ to all components of $R$ as $\typer{R}{\schema}$.
 To simplify readability, we also seldomly use notation $\typed{R}(\df_1,\ldots,\df_n)$ as a
shortcut for $\typed{R} = \tup{R/n,\tup{\df_1,\dots,\df_n}}$.

%\end{definition}

Obviously, since relations are typed, it is important to define when their
tuples agree with their facets.
%
%\begin{definition}
Let $\typed{R} = \tup{R/n,\dfset_R}$ be a relation schema. We say that a fact
$R(\cname{o_1},\ldots,\cname{o_n})$ \emph{conforms to $\typed{R}$} if for every
$i \in \set{1,\ldots,n}$, we have that $\cname{o_i}$ belongs to $\df_i$.
%\end{definition}
%%
%\begin{definition}
Let $\dfset$ be a set of facets, and $\schema$ be an $\dfset$-typed
database schema. A database instance $I$ \emph{conforms to} $\schema$
if every tuple $R(\cname{o_1},\ldots,\cname{o_n}) \in I$ conforms to
its corresponding relation schema $\typed{R} \in \schema$.
%\end{definition}

\subsection{Initial Data Domain}

Giving a data type $\dt = \tup{\tdom{\dt},\rel{\dt}}$, we
isolate a \emph{finite} subset $\cdom{\dt} \subset \tdom{\dt}$ of
\emph{initial data objects} for $\dt$. This subset explicitly enumerates those
data objects that can be used in the initial states of the agent
specifications (cf.~Section~\ref{sec:spec}), plus specific ``control
data objects'' that are explicitly
mentioned in the agent specifications themselves, and consequently contribute to
determine the possible executions.

We extend
this notion to cover also those objects used in the definition of
facets. Giving a facet $\df = \tup{\dt, \varphi(x)}$ with $\dt = \tup{\tdom{\dt},\rel{\dt}}$, the set of \emph{initial
data objects} for $\df$ %, written $\cdom{\df}$,
 is a finite subset of
$\tdom{\dt}$ that contains all data objects explicitly mentioned in
$\varphi(x)$.
The \emph{initial data domain} of an RMAS with set $\dfset$ of facets,
written $\cdom{\dfset}$,
is then defined as the (disjoint) union of initial data objects for
each facet in $\dfset$.

% Without loss
% of generality, we assume that $\cdom{\dt}$ also contains \emph{all} data
% objects that are explicitly mentioned in the facet formulae of facets
% defined starting from $\cdom{\dt}$.

% Giving a set $\dtset$ of data types, $\cdom{\dtset}$ denotes the
% disjoint union of all the initial data objects isolated for the data
% types in $\dtset$. As before, giving a set $\dfset$ of facets, we use
% notation $\cdom{\dfset}$ as a shortcut for $\cdom{\dtset}$, where
% $\dtset$ is the set of data types on which facets in $\dfset$ are
% defined.
% From now on, we assume that is set is fixed when the
% RMAS is defined.

\subsection{Typed Service Calls}
Typed service calls provide an abstract mechanism for agents to
incorporate new data objects when updating their own
databases. As argued in
  \cite{BCDDM13,MoCD14,BCDM14}, this is crucial to make the system ``open'' to
  the external world, and accounts for a variety of interaction modes,
  such as interaction with services or humans.
  We exploit this mechanism to model in particular the agent ability to inject new data according to internal decisions taken
  by the agent itself, but still external to its specification.

%\begin{definition}
Given a set $\dfset$ of facets, an \emph{$\dfset$-typed service}
$\typed{\scname{f}}$ is a triple $\tup{\scname{f}/n,\dfset^{in},\df^{out}}$, where
\begin{inparaenum}[\it (i)]
\item $\scname{f}/n$ is a function
schema with name $\scname{f}$ and arity $n$;
\item $\dfset^{in}$ is an $n$-tuple
$\tup{\df_1,\dots,\df_n}$ of facets in $\dfset$ representing the \emph{input
types} of the service call;
\item $\df^{out}$ is a facet in $\dfset$ representing the \emph{output
  facet} of the service call.
\end{inparaenum}
%\end{definition}
As for typed relations, in $\scset$ there are no two
typed services that share the same name.
Intuitively, when invoked with a tuple of ground data objects belonging
to their input facets, the service non-deterministically
returns a data object that belongs to the output facet.

\begin{example}
\label{ex:price}
 Service
  $\typed{\scname{getPrice}} = \tup{\scname{getPrice}/0,\set{SF},PF}$
  gets a string in $SF = \tup{\tup{\mathbb{S},\set{=}},\true}$
 referring to a product, and returns a rational price $PF
  = \tup{\tup{\mathbb{Q},\set{<,=}},x > \cname{0}}$ .
\end{example}

\begin{example}
\label{ex:services}
Given facet $AF = \tup{\tup{\mathbb{A},\set{=}},\true}$, service
  $\typed{\scname{getN}} = \tup{\scname{getN}/0,\emptyset,AF}$
  returns agent names.
\end{example}

\subsection{Agent Specifications}
\label{sec:spec}
In RMASs, agent specifications consist of three main components.
The first is the data component, whose intensional part is a typed database
  schema with constraints; every agent adopting the same specification
  starts with the same initial extensional data, but during the
  execution it autonomously evoles by interacting with other
  agents and services. The second is a proactive behavior, constituted by a set of condition-action
   communicative rules that determine which messages can be emitted by the
  agent, together with their actual payload and target agent.
The third is a reactive behavior, constituted by
ECA-like update rules that determine how the agent
  updates its own data when a certain message with payload is received
  from or sent to another agent.

%\begin{definition}
Given a set $\dfset$ of facets with initial data domain $\cdom{\dfset}$, an \emph{$\dfset$-typed agent
  specification} is a tuple $\tup{\name,\schema,\constr,\idb,\crule,\uact,\urule}$,
where:
\begin{compactenum}
\item $\name \in \mathbb{B} \cap \cdom{\dfset}$ is the
  \emph{specification name}, which is assumed to be also part of the
  initial data domain.
\item $\schema$ is an $\dfset$-typed database schema. We assume
    that the schema is always equipped with a special unary relation
    $\typed{\myname}$, whose unique component is typed with $AF$, and
    that is used to keep track of the global name associated to the
  agent in the system.
\item $\constr$ is a finite set of database constraints over
  $\schema$, i.e., of domain-independent first-order formulae over
  $\schema$ and
  $\drel{\dfset}$, using only constants from $\cdom{\dfset}$.
\item $\idb$ is the \emph{initial agent state}, i.e., a database instance
  that conforms to $\schema$, satisfies all constraints in $\constr$,
  and uses only constants from $\idb$.
% \begin{inparaenum}[\it (i)]
% \item $\idb$ conforms to $\schema$.
% \item the set of data objects in $\idb$, i.e., its active
%   domain $\adom{\idb}$, is contained in $\cdom{\dfset}$;
% \item $\constr$ is \emph{satisfied by} $\idb$, i.e.,
%  every constraint in $\constr$ evaluates to true in $\idb$.
% \end{inparaenum}
\item $\crule$ is a set of \emph{communicative rules}, defined below.
\item $\uact$ and $\urule$ are sets of \emph{update actions} and \emph{update rules}, defined below.
\end{compactenum}
%\end{definition}

When clear from the context, we use the name of a component with
superscript the name of the specification to extract that component
from the specification tuple. For example, $\schema^\cv{n}$ denotes
the database schema above.

\smallskip
\noindent \textbf{Communicative rules.} These rules are used to
determine which messages with payload are enabled to be sent by the agent to
other agents, depending on the current configuration of the agent
database. When multiple ground messages with payload are
enabled, the agent nondeterministically chooses one of them, according
to an internal, black-box policy.

%\begin{definition}
A \emph{communicative rule} is a rule of the form 
\[Q(t,\vec{x}) ~\ena~ M(\vec{x}) ~\toa~ t\]
where:
\begin{inparaenum}[\it (i)]
\item $Q$ is a domain-independent FO query over $\schema$ and $\drel{\dfset}$, whose
  terms are variables $t$ and
$\vec{x}$, as well as data objects in $\cdom{\dfset}$;
\item $M(\vec{x})$ is a message, i.e., a typed relation whose schema belongs to $\evset$.
\end{inparaenum}
%\end{definition}

%\begin{definition}
Let $\dfset$ be a set facets, $\schema$ a $\dfset$-typed
database schema, $D$ a database instance that conforms to $\schema$, and $Q(x_1,\ldots,x_n)$ a FO query over $\schema$ and
$\drel{\dfset}$ that uses only constants in $\cdom{\dfset}$. The \emph{answer $\ans{Q, D}$ to $Q$ over $D$}
is the set of assignments $\theta$ from the free variables $\vec{x}$
of $Q$ to data objects in $\cdom{\dfset}$, such that $D \models Q\theta$.
%\end{definition}
We treat
$Q\theta$ as a boolean query, and we say
$\ans{Q\theta,D} \equiv \true$ if and only if $D \models Q\theta$.

In the following, we use the special query $\inadom_\dt(x)$ as a
shortcut for the query that returns all data
objects in the current active domain that belong to data type
$\dt$. Given schema $\schema$, such a query can be easily expressed as
the union of conjunctive queries checking whether $x$ belongs to a
component of some relation in $\schema$, such that the component has
type $\dt$.   In this respect, notice that any query can be relativized
to the active domain through $\inadom$ atoms.

%  we use the abbreviation $\inadom(x)$ to check whether
% $x$ is present in the active domain of the agent. Given schema
% $\schema$, $\inadom(x)$ can be written as a union of atomic
% queries, each checking whether $x$ belongs to some component of a
% relation in $\schema$. We also use $\inadom(x_1,\ldots,x_n)$ as a shortcut
% for $\bigwedge_{i \in \set{1,\ldots,n}} x_i$.

% In this light, notice
% that any first-order query $Q(x_1,\ldots,x_n)$ can be relativized to the
% current active domain by writing the domain-independent query
% $(\bigwedge_{i \in \set{1,\ldots,n}} \inadom_{\outtype{x_i}}(x_i) )\land Q(x_1,\ldots,x_n)$.

We also make use to the anonymous variable ``$\_$'' to signify an
existentially quantified variable not used elsewhere.

\smallskip
\noindent \textbf{Update actions}. These are parametric actions used
to update the agent current database instance, possibly injecting new
data objects by interacting with typed services.

%\begin{definition}
An \emph{update action} is a pair
$\tup{\typed{\alpha},\alpha_{spec}}$, where:
\begin{inparaenum}[\it (i)]
\item
 $\typed{\alpha}$ is the \emph{action schema}, i.e., a typed relation accounting for the action name
and for the number of action parameters, together with their
types;
\item
$\alpha_{spec}$ is the action specification and has the form
$\alpha(\vec{\pname{p}}) : \{e_1,\ldots,e_n\}$, where
%$\alpha(\vec{\pname{p}})$ is the action signature with parameters $\vec{\pname{p}}$, and
$\{e_1,\ldots,e_n\}$ are update effects.
\end{inparaenum}
Each update effects has the form
\[\map{Q(\vec{\pname{p}},\vec{x})}{\add~A,\del~D}\]
where
  \begin{inparaenum}[\it (i)]
  % \item $q^+$ is a union of conjunctive queries (UCQ) over $\schema$ and
  % $\drel{\dfset}$, whose terms are
  %   parameters $\vec{\pname{p}}$, variables $\vec{x}$, and data objects from $\cdom{\dfset}$.
  % \item $Q^-$ is a
  %   FO query over $\schema$ and
  % $\drel{\dfset}$, whose free
  %   variables occur all among those of
  %   $q^+$; intuitively, $q^+$ extracts tuples from the
  %   current database instance, while $Q^-$ filters away some of
  %  them, guaranteeing domain-independence in our effects.
 \item $Q$ is a domain-independent FO query over $\schema$ and $\drel{\dfset}$, whose
  terms are parameters $\vec{\pname{p}}$, variables $\vec{x}$, and data objects in $\cdom{\dfset}$;
\item $A$ is a set of ``add'' facts over $\schema$ that
    include as terms: free variables $\vec{x}$ of $Q$,
    parameters $\vec{\pname{p}}$ and terms
    $\scname{f}(\vec{x},\vec{\pname{p}})$, with $\typed{\scname{f}}$
  in $\scset$;
\item $D$ is a set of ``delete'' facts that include as terms free variables
  $\vec{x}$ and parameters $\vec{\pname{p}}$.
\end{inparaenum}
%\end{definition}
%We use notation $\calls{A}$ to extract the set of service calls
%contained in the set $A$ of facts.

An update action is applied by grounding its parameters $\vec{\pname{p}}$
with data objects $\vec{\cname{o}}$. This results in partially grounding each of its
effects. The effects are then applied in parallel over the agent
database, as follows. For each partially grounded effect
$\map{Q(\vec{\pname{o}},\vec{x})}{\add~A,\del~D}$,
$Q(\vec{\pname{o}},\vec{x})$ is evaluated over the current database
and for each obtained answer $\theta$, the fully ground facts
$A\theta$ (resp., $D\theta$) are obtained. All the ground facts in
$D\theta$ are deleted from the agent database. Facts in
$A\theta$, instead, could contain (ground) typed service
calls. In this case, every service call is issued, obtaining back a (possibly fresh) data
object belonging to the output facet of the service. The instantiated
facts in $A\theta$ obtained by replacing the ground service calls with
the corresponding results are then added to the current database,
giving priority to additions.

\smallskip
\noindent \textbf{Update rules.}
These are conditional, ECA-like rules used by the agent to invoke an
update action on its own data when a message with payload is exchanged
with another agent.

%\begin{definition}
An \emph{update rule} is a rule of the form
\begin{compactitem}
\item (on-send) $\onev~ M(\vec{x})~\toa~t~\ifc~ Q(\vec{y}_1) ~\thendo~
  \alpha(\vec{y}_2)$, with $\vec{y}_1 \cup \vec{y}_2\subseteq \vec{x}
  \cup \set{t}$, or 
\item (on-receive) $\onev~ M(\vec{x})~\froma~s~\ifc~ Q(\vec{y}_1) ~\thendo~
  \alpha(\vec{y}_2)$, with $\vec{y}_1 \cup \vec{y}_2\subseteq \vec{x}
  \cup \set{s}$,
\end{compactitem}
where:
\begin{inparaenum}[\it (i)]
\item $M(\vec{x})$ is a message, i.e., a typed relation whose schema belongs to $\evset$;
\item $Q$ is a FO query over $\schema$, whose terms are variables
  $\vec{y}_1$ and data objects in $\cdom{\dfset}$;
\item $\alpha$ is an update action in $\uact$, whose parameters are
  bound to variables $\vec{y}_2$.
\end{inparaenum}
%\end{definition}

\smallskip
\noindent
\textbf{Institutional Agent Specification.}
In an RMAS, an \emph{institutional} agent is dedicated to the
management of the system as a whole. Differently from DACMASs \cite{MoCD14},
we do not assume here that the institutional agent has full visibility of
the messages exchanged by all agents acting into the system. It is
simply an agent that is always active in the system and whose name, $\inst$ in the following, is known by every other agent.
Still, we assume that the institutional agent has special duties, such as in
particular handling the creation of agents and their removal from the
system, and maintainance of agent-related information, like the set of
names for active agents, together with
their specifications.

Technically, the institutional agent specification $\instspec$ is a
standard agent specification named $\cname{ispec}$, partially grounded
as follows.
%%
%an
%agent specification
%$\tup{\cname{ispec},\schema^\cv{ispec},\constr^\cv{ispec},\idb^\cv{ispec},\crule^\cv{ispec},\uact^\cv{ispec},\urule^\cv{ispec}}%$,
%where we partially fix its different components as follows.
%%
%%
%\smallskip
%\noindent \textbf{Institutional database.}
To keep track of agents and their specifications, $\schema_i$ contains
three dedicated typed relations:
\begin{inparaenum}[\it (i)]
\item $\tup{\agent/1,\tup{\afacet}}$, to store
  agent names;
\item $\tup{\spec/1,\tup{\sfacet}}$, to store specification names;
\item $\tup{\hasSpec/2,\tup{\afacet,\sfacet}}$, to store
  the relationship between agents and their specifications.
\end{inparaenum}
Given these special relations, $\inst$
can also play the role of \emph{agent registry}, supporting agents in
finding names of other agents to communicate with.
Additional system-level relations, such as agent roles, duties, commitments
\cite{MoCD14}, can be insterted into $\schema_\inst$ depending on the
specific domain under study.
To properly enforce that $\hasSpec/2$ relates agent to specification
names, foreign keys can be added to $\constr^\cv{ispec}$.
% $\constr^\cv{ispec}$:
% \[
% \small
% \begin{array}{l}
% \forall a,s. \hasSpec(a,s) \limp \agent(a) \land \spec(s)\\
% \land \forall a. \agent(a) \limp \hasSpec(a,\_)
% \end{array}
% \]
Futhermore, we properly initialize $\idb^\inst$ as follows:
\begin{inparaenum}[\it (i)]
\item $\agent(\inst) \in \idb^\cv{ispec}$;
\item $\spec(s_i) \in \idb^\cv{ispec}$ for every agent specification that
  is part of the RMAS, i.e., for specification name $\cname{ispec}$
  and all specification names mentioned in $\agset$;
\item $\hasSpec(\inst,\cname{instSpec}) \in \idb^\cv{ispec}$.
\end{inparaenum}
Obviously, $\inst$ may have other initial data, and
specific rules and actions. Of particular interest is the possibility
for $\inst$ of dynamically creating and removing other agents. This
can be encoded by readapting \cite{MoCD14}. Details
are given in the online appendix.

\smallskip
\noindent\textbf{Agent creation/removal.} Two actions are employed by the
institutional agent to insert or remove an agent into/from the
system. Their respective action schemas are
$\actname{newAg}(\sfacet)$ and $\actname{remAg}(\afacet)$. As for creation, $\inst$ employs the service call
introduced in Example \ref{ex:services} to introduce a name into
the $\agent$ relation, attaching to it the specification name passed
as input. However, some additional modeling effort is
needed so as to ensure that the introduced name is indeed new:
\[
\actname{newAg}(\pname{s}) :
\left\{
\begin{array}{@{}r@{\ }c@{\ }l@{}l@{}}
\ex{OldAg}(a) & \rightsquigarrow &\del&\set{\ex{OldAg}(a)}\\
\ex{FreshAg}(a) & \rightsquigarrow &\del&\set{\ex{FreshAg}(a)}\\
\true & \rightsquigarrow & \add &
\left\{
\begin{array}{@{}l@{}}
\ex{FreshAg}(\scname{getN}()),\\
\agent(\scname{getN}()),\\
\spec(\scname{getN}(),\pname{s})
\end{array}
\right\}\\
\agent(a) & \rightsquigarrow & \add& \set{\ex{OldAg}(a)}
\end{array}
\!\!\right\}
\]
Intuitively, apart from adding the new agent and attaching the
corresponding specification, the action updates the two accessory
agent relations $\ex{OldAd}$ and $\ex{FreshAg}$, which are assumed to
be part of $\schema^\cv{ispec}$, ensuring that in the next state
$\ex{OldAd}$ contains the set of agent names that were present in the
immediately preceding state, and that $\ex{FreshAg}$ contains the
newly injected name. Freshness can then be guaranteed by adding a
dedicated constraint to $\Gamma^\cv{ispec}$:
\[
\forall a. \ex{OldAg}(a) \land \ex{FreshAg}(a) \limp \false
\]
Removal of an agent is instead simply modelled as:
\[
\actname{remAg}(\pname{a}) :
\left\{ \hasSpec(\pname{a},s) \rightsquigarrow \del
\left\{
\begin{array}{@{}l@{}}
\agent(\pname{a}),\\
\hasSpec(\pname{a},s)
\end{array}
\right\}
\right\}
\]
Update rules that employ these special actions obviously depend on the
domain, by including specific on-send and on-receive rules in $\instspec$.

% $\ev{ev}/ |\vec{x}|$ is an event type from $\evset$, $Q$ is an $\ecql$, and
%$\alpha$ is an update action with parameters (described
%below). Each such rule triggers when an event is sent to/received
%from another agent, and  $Q$ holds.
% rules trigger when a message is exchanged by two agents regarding a
% specific event with payload. A message exchange
% represents a synchronization point between the involved sender and receiver,
% which react by applying the corresponding on-send and on-receive rules
% in the same execution step.
%This results in the application of $\alpha$ using
%the actual event payload and receiver/sender. Action $\alpha$ queries the ABox of the agent and of $\inst$,
%using the answers to add and remove facts to the ABox. %When adding

%%% Local Variables:
%%% mode: latex
%%% TeX-master: "main"
%%% save-place: t
%%% End:

\subsection{Well-Formed Specifications}
In an RMAS, every piece of
information is typed. This immediately calls for a suitable notion of
\emph{well-formedness} that checks the compatibility of types in all agent
specifications. Intuitively, an RMAS $\sys$ is well-formed if:
\begin{inparaenum}[(1)]
\item every query appearing in $\sys$ consistently use variables, that
  is, if a variable appears multiple components, they all have the
  same data type;
\item every proactive rule instantiates the message payload with
  compatible data objects, and the destination agent with an agent name;
\item every reactive rule correctly relates the data types of the
  message payload with those of the query and of the update action;
\item each action effect uses parameters in a compatible way with the action type;
\item each action effect instantiates the facts in the head in a
  compatible way with their types;
\item each service call correctly binds its inputs and output.
\end{inparaenum}

We now formalize this intuition.
Let $\dfset$ be a set facets, and $\schema$ be a $\dfset$-typed
database schema. Let $Q$ be a FO query over $\schema$ and
$\drel{\dfset}$ that uses only constants in $\cdom{\dfset}$. We say
that $Q$ is \emph{$\schema$-compatible} if:
\begin{inparaenum}[\it (i)]
\item whenever a data object from $\cdom{\dfset}$ appears in 
  component $R[i]$ inside $Q$, then it belongs to $\type{R}{i}{\schema}$;
\item whenever the same variable $x$ appears in two components $R_1[i_1]$ and
  $R_2[i_2]$, then $\type{R_1}{i_1}{\schema} = \type{R_2}{i_2}{\schema} $.
\end{inparaenum}

By definition of compatibility, each free variable of a
$\schema$-compatible query is associated to a single facet/data
type. This allows us to characterize the ``output types'' of a query,
that is, the types associated to its free variables (and hence also
the types of its answer components). Given an $\dfset$-typed
database schema $\schema$ and  a well-formed FO query $Q(\vec{x})$ over $\schema$ and
$\drel{\dfset}$ that uses only constants in $\cdom{\dfset}$, the
\emph{output-type of $x_i \in \vec{x}$ according to $Q$}, written $\outtypec{x_i}{Q}$, is the unique data type
in $\dfset$ to which $x_i$ is associated by $Q$,  where $\dtset$
is the set of data types on which $\dfset$ is defined. We extend the
notion of output-type to a tuple of variables $\vec{x}' =
\tup{x_{i},\ldots,x_{k}} \subseteq \vec{x}$ with $1 \leq i \leq k
\leq n$, writing $\outtypec{\vec{x}'}{Q}$ as a shortcut for the tuple
$\tup{\outtypec{x_i}{Q},\ldots,\outtypec{x_k}{Q}}$. We also write
$\outtype{Q}$, as a shortcut for $\outtypec{\vec{x}}{Q}$. Notice that,
when applied to an atomic query, this notion corresponds exactly to
the typing of the corresponding relation, according to its schema.

Given an RMAS $\sys = \rmas$ and an agent specification
$\N = \tup{\name,\schema,\constr,\idb,\crule,\uact,\urule}$ in $\agset \cup
\set{\instspec}$, we say that:
\begin{compactitem}
\item $\crule$ is \emph{well-formed} if each of its communicative
  rules $Q(t,\vec{x})
~\ena~ M(\vec{x}) ~\toa~ t$
is such that 
\begin{inparaenum}[\it (i)]
\item  $\outtypec{t}{Q} = \mathbb{A}$ (i.e., $Q$ binds $t$ to an
agent name), and
\item $\outtype{Q} = \typer{M}{\evset}$ (i.e., the payload is
instantiated by $Q$ in a compatible way with the types of message $M$).
\end{inparaenum}
\item $\uact$ is \emph{well-formed} if each of its actions is
  well-formed. We say in turn that action $\tup{\typed{\alpha},\alpha_{spec}}$
  is \emph{well-formed} if every effect $\map{Q(\vec{\pname{p}},\vec{x})}{\add~A,\del~D}$
in $\alpha_{spec}$ is such that:
\begin{compactitem}
\item $Q$ is $\schema$-compatible.
\item Whenever a parameter $\pname{p}$ is mentioned in $Q$, the type
  to which $\pname{p}$ is assigned by $\outtype{Q}$ is the same to
  which $\pname{p}$ is assigned by $\typed{\alpha}$.
\item For every $n$-ary typed relation $\typed{R} \in \schema$, every fact $F$ of
  $R$ appearing in $D$, and for each $i \in \set{1,\ldots,n}$:
\begin{inparaenum}[\it (i)]
\item if the $i$-th position of $F$ contains a data object, then such data
  object belongs to the domain of $\type{R}{i}{\schema}$;
\item  if the $i$-th position of $F$ contains a variable $y \in
  \vec{x}$, then $\outtypec{y}{Q} = \type{R}{i}{\schema}$.
\end{inparaenum}
\item For every $n$-ary typed relation $\typed{R} \in \schema$, every fact $F$ of
  $R$ appearing in $A$, and for each $i \in \set{1,\ldots,n}$:
\begin{inparaenum}[\it (i)]
\item if the $i$-th position of $F$ contains a data object, then such data
  object belongs to the domain of $\type{R}{i}{\schema}$;
\item  if the $i$-th position of $F$ contains a variable $y \in
  \vec{x}$, then $\outtypec{y}{Q} = \type{R}{i}{\schema}$;
\item if the $i$-th position of $F$ contains a $k$-ary service call
  $f(\vec{y})$ with $\vec{y} \subseteq \vec{x}$ and
  $\tup{f/k,\dfset^{in},\df^{out}} \in \scset$, then
  $\outtypec{\vec{y}}{Q} = \dfset^{in}$ and $\df^{out} = \type{R}{i}{\schema}$.
\end{inparaenum}
\end{compactitem}
\item $\urule$ is \emph{well-formed} if all its update rules
  are well-formed. We discuss the case of on-send rules - the
  definition of well-formedness is identical for on-receive rules. An
  on-send rule in $\urule$ of the form
$\onev~ M(\vec{x})~\toa~t~\ifc~ Q(\vec{y}_1) ~\thendo~
  \alpha(\vec{y}_2)$, with $\vec{y}_1 \cup \vec{y}_2\subseteq \vec{x}
  \cup \set{t}$, is \emph{well-formed} if the following conditions hold:
\begin{inparaenum}[\it (i)]
\item if $t \in \vec{y_1}$, then $\outtypec{t}{Q} = \mathbb{A}$;
\item if $t$ appears in the $i$-th component of $\alpha$, then
  $\typed{\alpha}$ assigns type $\mathbb{A}$ to its $i$-th parameter;
\item for each variable $x \in \vec{x} \cap \vec{y}_1$, such that $x$
  appears in the $i$-th component of $M$, we have that
  $\type{M}{i}{\evset} = \outtypec{x}{Q}$;
\item for each variable $x \in \vec{x} \cap \vec{y}_2$, such that $x$
  appears in the $i$-th component of $M$ and in the $j$-th component
  of $\alpha$, we have that $\typed{\alpha}$ assigns type 
  $\type{M}{i}{\evset}$ to its $j$-th parameter.\end{inparaenum}
\item $\N$ itself is \emph{well-formed} if $\crule$, $\uact$ and
  $\urule$ are all well-formed.
\end{compactitem}
Finally, we say that the entire RMAS $\sys$ is \emph{well-formed} if
all agents specifications in $\agset \cup \set{\instspec}$ are well-formed.

 It is easy
to see that checking whether an RMAS is well-formed requires linear
time in the size of the specification.

% Specifically, the following sources of ill-typedness
% may arise inside an agent specification:
% \begin{compactitem}
% \item a query joins incompatible relation components;
% \item a proactive rule wrongly instantiates the message payload;
% \item a reactive rule wrongly relates the types of the message payload
%   and those of the invoked update action;
% \item the head of an effect wrongly binds the inputs or output of a
%   service call.
% \end{compactitem}
% We precisely capture all such compatibility requirements in the
% following definitions.
%%
From now on, we always assume that RMASs are well-formed.
%%
% \begin{definition}
% Given a set $\dfset$ of facets, we inductively define the
% \emph{$\dfset$-compatibility} of a FO query $Q(\vec{x})$, written  as follows:
% \[
% \Huge \begin{array}{@{}l@{~}c@{~~}l}
% TBD
% \end{array}
% \]
% \end{definition}
%%
%%
It is important to notice that well-formedness does not guarantee that
the restrictions imposed by facets are always satisfied, but only that
the agent specification consistently use data types. Consistency with
facets is managed at runtime, by dynamically handling facet violations
(cf.~Section~\ref{sec:ts}).

%%% Local Variables:
%%% mode: latex
%%% TeX-master: "main"
%%% save-place: t
%%% End:

\section{Modeling with RMAS}
\label{sec:example}

We briefly show how RMASs can be easily accommodate complex data-aware
interaction protocols, leveraging on data types. We take inspiration from
ticket-based mutual exclusion protocols \cite{BuGP99,BaKa08}. This can be used,
in our setting, to guarantee the possibility for an agent to engage in a
complex, critical interaction with the institutional agent.

Another interesting example, namely how to model a form of
\emph{contract net protocol} in RMASs, is provided in Section~\ref{sec:cn}. The interested reader
can also refer to \cite{MoCD14} for commitment-based interactions.

From now on, we assume that interaction in RMAS is synchronous. This
assumption is without loss of generality, since message queues for
asynchronous communication can be modelled as special typed relations
in the agent databases:

\begin{theorem}
\label{thm:asynch}
Asynchronous RMASs based on message queues can be simulated by
synchronous RMASs.
\end{theorem}

\begin{proof}
We consider a form of reliable, asynchronous communication based on
message buffers. In particular, the model works as follows:
\begin{compactitem}
\item Messages sent by an agent to itself are processed immediately
  (in fact, there is no effective communication in this case).
\item Whenever a sender agent emits a message with payload targeting
  another agent, the message is atomically inserted into a message
  buffer attached to the target agent.
\item The target agent asynchronously reacts to the message by
  extracting it from the buffer (this could happen much later).
\item We consider two variations of this general model: one in which
  the buffer is ordered (i.e., it is a queue), and one in which the buffer
  is just a set of messages. Both models are interesting, because they reflect
  different assumptions on the asynchronous communication model. In fact, the
  first guarantees that the order in which messages are processed by
  the target follows the order in which messages where emitted
  (possibly by different agents). We call this communication model
  \emph{asynchronous, ordered} (AO for short), and use acronym AO-RMAS
  for an RMAS
  adopting the AO communication model.
Contrariwise, the second model
  accommodates the situation in which the order in which messages are
  received (i.e., processed) by a target agent does not necessarily
  reflect the order in which such messages were emitted. We call this communication model
  \emph{asynchronous, disordered} (AD for short), and use acronym AD-RMAS
  for an RMAS
  adopting the AD communication model.
\end{compactitem}
We prove that these asynchronous communication models can be both
accommodated by a synchronous RMAS that employs accessory data
structures in the agent schemas, specifically tailored to buffer
messages and decouple the
emission of a message from its processing by the target agent.

Given an AD-RMAS/AO-RMAS $\sys = \rmas$, we convert it into a standard,
synchronous RMAS $\sys_s = \tup{\dtset, \dfset, \cdom{\dfset},\scset_s, \evset_s,
    \agset_s,\instspec_s}$, where $\evset_s$ and $\scset_s$ just
  extend $\evset$ and $\scset$
  with an additional message/service as illustrated below, and where each agent specification $\N = \tup{\name,\schema,\constr,\idb,\crule,\uact,\urule}$ in $\agset \cup
\set{\instspec}$ becomes a corresponding agent specification $\N_s = \tup{\name,\schema_s,\constr_s,\idb,\crule_s,\uact_s,\urule_s}$ in $\agset_s \cup
\set{\instspec_s}$, according to the translation
mechanism illustrated in the following (notice that we are interested here in the correctness of the
  encoding, not in its efficient implementation; effective ways of
  realizing the translation can be provided by using this encoding as
  a basis).

Let us first focus on the database schema of the agent
specification. We set $\schema_s = \schema \cup \set{\typed{\buff},\typed{\newm},\typed{\oldm}}$, where $\typed{\buff}$ is a  global buffer tracking incoming messages that have been received
  by the agent but still needs to be (asynchronously) processed, while
  $\typed{\newm}$ and $\typed{\oldm}$ are unary accessory relations used
  to manage the generation of new identifiers for messages to be
  enqueued. The management of such identifiers closely resembles name
  management as discussed for the institutional agent. 

Specifically,
  $\typed{\buff}$ contains a numeric primary key, and internalizes the
  payload schemas of all message relations in
  $\evset$, plus an additional component to track the sender agent,
  and a boolean component indicating whether the payload has a valid content. A tuple
  in $\typed{\buff}$ contains a message identifier and sets exactly
  one of such boolean components to $\true$, leaving the others 
  $\false$. This indicates what is the type of the buffered message, and
  that the corresponding payload/sender components contain the actual
  message payload and sender agents, whereas all other payload/sender components contain meaningless
  values. For this latter aspect, we assume, without loss of
  generality, that all data types are equipped with an $\cv{undef}$ined data
  object. 

Technically, we fix an ordering over $\evset$, that is, a
  bijection 
\[msg: \set{1,\ldots,|\evset|} \longrightarrow \evset \]
and fix the function $index = msg^{-1} $.
We set the arity of $\typed{\buff}$ to $1+\sum_{i=1}^{|\evset|} (2 + a_i)$,
where $a_i$ is the arity of relation $msg(i)$. We make use of the
following three specific types:
\begin{inparaenum}[\it (i)]
\item  the $Bool$ facet (cf.~Example~\ref{ex:bool}),
\item the $RF$
facet, defined as $\tup{\tup{\mathbb{R},\set{=,<}},\true}$, and
\item the facet $AF$ for agent names.
\end{inparaenum}
 Specifically,
 we type component $\buff[1]$ (the relation primary key) with
$RF$. For each $i \in  \set{1,\ldots,|\evset|}$, we
type component $\buff\left[2+\sum_{j=1}^{i-1} (2 + a_j)\right]$ with
$Bool$, and component $\buff\left[3+\sum_{j=1}^{i-1} (2 + a_j)\right]$
with $AF$, where $a_i$ is the arity of $msg(i)$. Furthermore, for each $i \in
\set{1,\ldots,|\evset|}$ and for every $k \in \set{1,\ldots,a_i}$
($a_i$ being the arity of $msg(i)$), we set the type of component
$\buff\left[3+\sum_{j=1}^{i-1} (2 + a_j) + k\right]$ to be the same as
the type of component $msg(i)[k]$.

Unary relations $\typed{\newm}$ and $\typed{\oldm}$ are respectively
used to store newly created or already existing message
identifiers. Their unique component is consequently typed with
$RF$.

Let us now consider the database constraints. We set $\constr_s =
\constr \cup \set{\Phi_{msgId}}$, where $\Phi_{msgId}$ is a
constraint ensuring that new message identifiers do not clash with
already existing identifiers, and whose specific shape depend on
whether the original RMAS is asynchronous ordered or unordered. In
particular:
\begin{compactitem}
\item if  $\sys$ is an AD-RMAS, then $\Phi_{newMsg}$ is
\[
 \forall id_n,id_o. \newm(id_n) \land \oldm(id_o) \limp id_o \neq id_n
\]
(where $id_o \neq id_n$ is an abbreviation for $\neg (id_o = id_n)$).
\item if  $\sys$ is an AO-RMAS, then $\Phi_{newMsg}$ is
\[
 \forall id_n,id_o. \newm(id_n) \land \oldm(id_o) \limp id_o < id_n
\]
In fact, for an ordered RMAS, a newly created message must be enqueued
after all pending messages that were enqueued before. 
\end{compactitem}

We now focus on the behavior of $\N_s$, that is, on how the rules of
$\N$ are translated into corresponding rules in $\N_s$ so as to
simulate the asynchronous communication model on top of a synchronous
communication model. Since asynchronous communication requires to decouple the
emission of a message from the reaction of the target agent,
$\urule_s$ only maintains the on-send rules of $\urule$, replacing the
on-receive rules with other on-receive rules. This new on-receive
rules are organized in two groups. The first group of rules is just
used to insert message received from other agents into the buffer. In particular, for each $i \in \set{1,\ldots,|\evset|}$,
$\urule_s$ contains a rule of the form
\[
\onev~ M_i(\vec{x})~\froma~s~\ifc~ \neg \myname(s) ~\thendo~
  \actname{buffer}_{M_i}(\vec{x},s)
\]
where $M_i$ is the name of relation $msg(i)$, and
$\actname{buff}_{M_i}$ is a specific update action in $\uact_s$,
dedicated to insert the payload and sender agent of a message $M_i$
into the buffer. In particular,
$\actname{buff}_{M_i}(\vec{\pname{x}},\pname{s})$ is defined as: 
\[
\left\{
\begin{array}{@{}l@{}}
\oldm(m)  \rightsquigarrow \del\set{\oldm(m)}\\
\newm(m)  \rightsquigarrow \del\set{\newm(m)}\\
\true  \rightsquigarrow  \add 
\left\{
\begin{array}{@{}l@{}}
\newm(\scname{getRN}()),\\
\buff(\scname{getRN}(),\ldots,\underbrace{``\cv{t}"}_{\makebox[0pt]{\scriptsize
    i-th component}},\pname{p},\vec{\pname{x}},\ldots)\\
\end{array}
\right\}\\
\buff(m,\_,\ldots,\_)  \rightsquigarrow \add \set{\oldm(m)}
\end{array}
\!\!\right\}
\]
where $\typed{\scname{getRN}}$ is a service that returns a $RF$ data
object, and in the addition of the tuple
$\buff(\scname{getRN}(),\ldots,``\cv{t}",\pname{p},\vec{\pname{x}},\ldots)$,
attributes $``\cv{t}",\pname{p},\vec{\pname{x}}$ are inserted in those
positions corresponding to the boolean component, sender agent
component, and payload components dedicated to $msg(i)$, while all the
other boolean components are set to $``\cv{f}"$, and all remaining
components are set to $\cv{undef}$.

The processing of a buffered message is triggered by a special
communicative rule that is contained in $\crule_s$ together with all
the original rules in $\crule$. The purpose of the communicative rule
is to extract a message from the buffer, triggering the agent to
process it whenever the original specification contained on-receive
rules dedicated to this. This is done by self-sending a message $\mathit{nextM}$
Specifically:
\begin{compactitem}
\item If  $\sys$ is an AD-RMAS, the message extraction rule is:
\[
\begin{array}{l@{~}l}
\myname(a)& \\{}\land ~\buff(m,\_,\ldots,\_)&\ena~ nextM(m) ~\toa~ a
\end{array}
\]
Indeed, for a disordered RMAS, the order in which messages are
received is non-deterministic. This rule mimics such a
nondeterminism, since the agent nondeterministically picks one of the
buffered messages. 
\item If  $\sys$ is an AO-RMAS, the message extraction rule is:
\[
\begin{array}{l}
\myname(a) \land \buff(m,\_,\ldots,\_) \\
\land \neg (\exists m_2. \buff(m_2,\_,\ldots,\_)
\land m_2 < m)  \\ \ena~ nextM(m) ~\toa~ a
\end{array}
\]
Indeed, for an ordered RMAS, messages are deterministically received
according to the order in which they have been sent. This rule mimics
such a determinism by following a FIFO policy, picking the first message
in the queue. Recall that, for AO-RMAS, whenever a new message is
inserted into the queue, its primary key is greater than the primary
keys of already enqueued messages. 
\end{compactitem}
The last dimension to be covered is the agent reaction to a message to
be processed. This is done by suitably reformulating the original on-receive
rules present in $\urule$. Specifically, for each on-receive rule
\[\onev~ M(\vec{x})~\froma~s~\ifc~ Q(\vec{y}_1) ~\thendo~
  \alpha(\vec{y}_2)\]
 in $\urule$, with $\vec{y}_1 \cup \vec{y}_2\subseteq \vec{x}
  \cup \set{s}$, $\urule_s$ contains a corresponding on-receive rule
  (which, by construction of $\sys_s$, is triggered only by the agent itself)
\[
\small
\begin{array}{@{}l@{}}
\onev~ \mathit{nextM}(m) ~\froma~a\\
\ifc~
\myname(a)
\land
\Phi_{M}(m,\vec{y}_1,\vec{y}_2)
\land Q(\vec{y}_1)
~\thendo~
\alpha(m,\vec{y}_2)
\end{array}
\]
where $\Phi_{M}(m,s,\vec{x})$ is a query that:
\begin{inparaenum}[\it (i)]
\item checks whether
the identifier $m$ points to a tuple in the buffer that actually
refers to a message of type $M$ (this can be done by checking whether
the boolean component in position $\mathit{index}(M)$ is set to $``\cv{t}"$);
\item if so, extracts the sender of message $m$, and its payload $\vec{x}$.
\end{inparaenum}
Technically, the query is simply formulated as:
\[
\Phi_{M}(m,s,\vec{x}) =
\buff(m,\_,\ldots,\underbrace{``\cv{t}"}_{\makebox[0pt]{\scriptsize$\mathit{index}(M)$-th component}},s,\vec{x},\ldots)
\]
A final, additional update rule that always triggers when a
$\mathit{nextM}$ message is received is needed to properly update the
buffer, by removing the processed message:
\[
\onev~ \mathit{nextM}(m) ~\froma~a~
\ifc~
\myname(a)
~\thendo~
\actname{removeM}(m)
\]
where:
\[
\small
\actname{removeM}(\pname{m}) :
\set{
\buff(\pname{m},\vec{x}) \rightsquigarrow \del\set{\buff(\pname{m},\vec{x})}
}
\]
By putting everything together, if we project away the accessory
relations $\buff$, $\oldm$ and $\newm$, we obtain that the asynchronous execution semantics of $\sys$
under both the ordered and disordered assumption exactly corresponds to
that of $\sys_s$ under the standard synchronous semantics, as
precisely defined in Figure~\ref{alg:ts}.
\end{proof}
The proof of Theorem~\ref{thm:asynch} already gives a glimpse about
the modelling power of RMASs equipped with ordered types. We next
discuss how these features can be exploited to easily capture mutual
exclusion protocols based on tickets.

\subsection{Ticket-Based Mutual Exclusion Protocols}
The idea behind ticket-based mutual exclusion protocols is that, when
a process wants to access a critical section, it
must get a ticket, and wait until its turn arrives. We model tickets
using the base facet  $RF =
\tup{\tup{\mathbb{R},\set{<,=}},\true}$ for real numbers, and exploit the
domain-specific relation $<$ to compare agent tickets. In our
formulation, the critical section consists of a (possibly complex) interaction
with the $\inst$, excluding the possibility for other
agents to concurrently engage in the
same kind of interaction with $\inst$.

We focus on the realization of $\inst$, in such a way that mutual
exclusion is guaranteed no matter how the other agents behave.
First of all, $\inst$ gives top priority to handle ticket
requests by the agents. A ticket request is issued by another agent
using a 0-ary message $\ev{askTicket}$. Agent $\inst$ reacts by
invoking a ticket generation action, provided that the sender agent is
not already owner of a ticket, and the $\ex{Assigned}$ relation is
empty (see below):
\[
\begin{array}{@{}l@{}}
\onev\ \ev{askTicket}() \ \froma\ a \\
\ifc\ \neg  \ex{HasTicket}(a,\_) \land \neg \ex{Assigned}(\_,\_) \ \thendo\ \actname{genTicket}(a)
\end{array}
\]
Action $\actname{genTicket}$ takes as input an agent name, and uses
a typed service $\typed{\scname{getTicket}} =
  \tup{\scname{getTicket}/0,\emptyset,RF}$ to get a
  numerical ticket. The result is stored in the temporary relation
  $\ex{Assigned}$, tracing that the ticket has been assigned
  but the corresponding agent still needs to be informed.
\[
\actname{genTicket}(\pname{a}): \{ \true \rightsquigarrow \add\set{\ex{Assigned(\pname{a},\scname{getTicket}())}}\}
\]
To guarantee that every agent will have the possibility of engaging
the critical interaction with $\inst$, every time a ticket is assigned
to an agent, $\inst$ must ensure that such agent will be served
\emph{after} those already possessing a ticket. This is enforced
through the following database constraint, which leverages on the
domain-specific relation $>$ for tickets:
\[
\forall t_{new},t. \ex{Assigned}(\_,t_{new}) \land
\ex{HasTicket}(\_,t) \limp t_{new} > t
\]
An assigned ticket must be sent to the requestor agent:
\[
\ex{Assigned}(t,a) \ \ena\ \ev{giveTicket}(t) \ \toa\ a
\]
to which $\inst$ itself reacts by moving the tuple from the temporary
relation $\ex{Assigned}$ to $\ex{hasTicket}$:
\[
\begin{array}{l}
\onev\ \ev{giveTicket}(t) \ \toa\ a
\ \ifc\ \true \ \thendo\ \actname{bindTicket}(a,t)\\[4pt]
\actname{bindTicket}(\pname{a},\pname{t}):
\left\{
\begin{array}{@{}r@{~}c@{~}l@{}l@{}}
\true &\rightsquigarrow& \del&\set{\ex{Assigned}(\pname{a},\pname{t})}\\
\true & \rightsquigarrow& \add&\set{\ex{hasTicket}(\pname{a},\pname{t})}
\end{array}
\right\}
\end{array}
\]
Now, let $\ev{cMsg}$ be a critical message.
% with payload
%$\vec{p}$, which is meant to trigger action $\ev{cAct}$ on the $\inst$
%side. We assume that the payload is augmented with the ticket
%owned by the agent.
To engage in the critical interaction with
$\inst$ triggered by message $\ev{cMsg}$, the agent provides the
payload and the ticket. Agent $\inst$ positively react to the request
provided that the ticket indeed corresponds to the agent, and that the
ticket is now to be served (i.e., it is smaller than any other
ticket):
\[
\small
\begin{array}{l}
\onev\ \ev{cMsg}(\vec{p},t) \ \froma\ a
\\ \ifc\ \ex{hasTicket}(a,t) \land \neg (\exists
a',t'. \ex{hasTicket}(a',t') \land  t > t')\\ \thendo\ \actname{cAct}(a,\vec{p})
\end{array}
\]
This pattern can be replicated for any other critical
interaction. Additional, state relations can be added to discipline
the orderings among critical message exchanges.

% Finally, message
% $\ev{endCI}$ can be sent by an agent to end the critical interaction
% with $\inst$, and release its ticket:
% \[
% \begin{array}{l}
% \onev\ \ev{endCI}(t) \ \froma\ a
% \\ \ifc\ \ex{hasTicket}(a,t) \land \neg (\exists
% a',t'. \ex{hasTicket}(a',t') \land  t > t')\\ \thendo\
% \actname{releaseTicket}(a,t)\\[4pt]
% \actname{releaseTicket}(\pname{a},\pname{t}):
% \{
% \true \rightsquigarrow \del\set{\ex{hasTicket}(\pname{a},\pname{t})}
% \}
% \end{array}
% \]

\subsection{Contract Net}
\label{sec:cn}
We show how the classical contract net protocol \cite{Smi80} can be easily
accommodated in our framework. This can be considered as an example of
a ``price-based'' protocol, and therefore indirectly shows how
different kinds of auctions could be modelled as well, as, e.g., done
in \cite{Bela14}.

An RMAS that incorporates the contract net protocol contains two agent
specifications (that can be obviously enriched and extended on a
per-domain basis): the specification of an \emph{initiator} agent, and
the specification of a \emph{participant} agent. The first
specification is embodied by an agent that is interested in delegating
the execution of a task to another agent, so as to achieve a desired
goal. The second specification is embodied by agents that have the
capabilities and the interest in executing the task, provided that
they get back a reward.  

% Here we consider the contract net to be initiated by the institutional
% agent. Obviously, this is not mandatory, and any other agent could
% start the protocol, after having acquired a list of candidate agents
% to be contacted.

The system employs the following FIPA-like messages:
\begin{compactitem}
\item $\typed{\ex{cfp}}(SF)$ (from the initiator to
  participants) -- a
  call-for-proposal related to the execution of the provided task (for
  simplicity, we use strings to represent tasks, and we assume that
  the task name is used also as a \emph{conversation identifier});
\item $\typed{\ex{propose}}(SF,PF)$ (from a participant to the initiator), with $PF$ as in
Example~\ref{ex:price}  -- a proposal to execute the task indicated in
the first parameter, for the price indicated in the second parameter;
\item $\typed{\ex{reject}}(SF)$ (from the
  initiator to a participant) -- rejection of
  all proposals for the specified task;
\item $\typed{\ex{accept}}(SF,PF)$ (from the
  initiator to a participant) -- acceptance of
  a proposal;
\item  $\typed{\ex{inform}}(SF)$ (from a
  participant to the initiator) -- notification that the task has been executed.
\item  $\typed{\ex{failure}}(SF)$ (from a
  participant to the initiator) -- notification that the task
  execution failed.
\end{compactitem}

Let us focus on the realization of the protocol from the point of view
of $\inst$, which acts as the initiator. We first introduce the
relations used by $\inst$ to run the protocol:
\begin{compactitem}
\item $\typed{\agent}(AF)$ lists the (names of) agents known to the
  initiator agent; if the initiator agent is $\inst$, then it already
  holds all agents present in the system, otherwise the initiator
  agent can engage in a preliminary interaction with $\inst$ and/or
  other agents to collect such names.
\item $\typed{\ex{Task}}(SF,\ex{StateF})$ lists the task names that
  the initiator agent is interested to assign, i.e., those that can
  become the subject of an instance of the contract net protocol. $\ex{StateF} =
  \tup{\tup{\mathbb{S},\set{=}},x = \str{todo} \lor x = \str{assigned}
    \lor x = \str{done}}$ is an enumerative facet used to
  track the state of each task -- the three states are self-explanatory.
\item $\typed{\ex{Contacted}}(AF,SF)$ lists those agents that have
  been already contacted for a given task.
\item $\typed{\ex{PropPrice}}(AF,SF,PF)$ lists those agents that
  answered to a proposal with a certain price.
\item $\typed{\ex{AssignedTo}}(AF,SF,PF)$ lists those tasks that have
  been assigned to an agent for a given price. 
\end{compactitem}

We have now all the ingredients to model the behavioral rules of the
initiator agent. First of all, the initiator agent can issue a
call-for-proposal for any task in the $\str{todo}$ state, directed
towards an eligible agent. This is captured by the communicative rule:
\[
\begin{array}{@{}l@{}}
\ex{Task}(t,\str{todo}) \land \agent(a) \\
{} \land \Phi_{sui}(a,t) \land \neg \ex{Contacted}(a,t)
~\ena~ \ex{cfp}(t) ~\toa~ a
\end{array}
\]
where $\Phi_{sui}(a,t)$ is a boolean query that checks whether $a$ is
a suitable agent for executing $t$, and that does so by possibly involving
additional relations maintained by the initiator agent for this
specific purpose. An agent is considered eligible if it is suitable
and has not been already contacted for the selected task.

The initiator agent reacts to this message by indicating that agent
$a$ has been contacted for task $t$:
\[
\onev~\ex{cfp}(t) ~\toa~ a~\ifc~
\true
~\thendo~
\actname{markContacted}(a,t)
\]
where
\[
\actname{markContacted}(\pname{a},\pname{t})
:
\big\{
\true \rightsquigarrow \add\set{\ex{Contacted}(\pname{a},\pname{t})}
\!\big\}
\]

When a proposer agent sends back a proposal, the initiator agent
stores it into the $\typed{\ex{PropPrice}}$ relation:
\[
\onev~\ex{propose}(t,p) ~\froma~s~\ifc~
\true
~\thendo~
\actname{setProposal}(s,t,p)
\]
where
\[
\actname{setProposal}(\pname{s},\pname{t},\pname{p})
:
\big\{
\true \rightsquigarrow \add\set{\ex{PropPrice}(\pname{s},\pname{t},\pname{p})}
\!\big\}
\]
Notice that this formalization seamlessly enables the same agent to make different
proposals for the same task, but can be easily modified so as to
account for the situation where only one proposal per agent can be accepted.

The presence of at least one registered proposal enables the
initiator to assign the task to some agent, provided that such an
agent made the best proposal, i.e., that with the lowest price. Notice
that the initiator is free to choose \emph{when} to accept, and can
decide to contact further agents before actually selecting the best
proposal. 
\[
\begin{array}{@{}l@{}}
\ex{PropPrice}(a,t,p) \land 
\neg (\exists p_2. \ex{PropPrice}(\_,t,p_2) \land p_2 < p )
\\\ena~ \ex{accept}(t,p) ~\toa~ a
\end{array}
\]
When the initiator decides to actually accept the best offer, it
reacts by tracking to which agent the task has been assigned (and with
wich price), taking also care of properly updating the task state, as
well as to clean the $\ex{PropPrice}$ relation. This is done through
two different rules. The task assignment is handled by rule
\[
\onev~\ex{accept}(t,p) ~\toa~ a~\ifc~
\true
~\thendo~
\actname{markAssigned}(a,t,p)
\]
where $\actname{markAssigned}(\pname{a},\pname{t},\pname{p})$
:
\[
\left\{
\begin{array}{@{}l@{\ }c@{\ }l@{\ }l@{}}
\true & \rightsquigarrow &
\add& \set{\ex{AssignedTo}(\pname{a},\pname{t},\pname{p})}\\
\ex{PropPrice}(\pname{a},\pname{t},p_a) & \rightsquigarrow &
\del& \set{\ex{PropPrice}(\pname{a},\pname{t},p_a)}
\end{array}
\right\}
\]
The task state update is instead managed by rule
\[
\onev~\ex{accept}(t,p) ~\toa~ a~\ifc~
\true
~\thendo~
\actname{setState}(t,\str{assigned})
\]
where $\actname{setState}(\pname{t},\pname{state})$ is a generic
state-update action formalized as follows:
\[
\left\{
\begin{array}{@{}l@{\ }c@{\ }l@{\ }l@{}}
\ex{Task}(\pname{t},\pname{oldstate}) & \rightsquigarrow &
\del&\set{\ex{Task}(\pname{t},\pname{oldstate})}\\
&&\add&\set{\ex{Task}(\pname{t},\pname{state})}\\
\end{array}
\right\}
\]

The acceptance of an offer enables the initiator to send a rejection
to all the agents that made an offer but were not selected:
\[
\begin{array}{@{}l@{}}
\ex{PropPrice}(a,t,\_) \land 
\neg (\ex{AssignedTo}(a,t,\_))
\\\ena~ \ex{reject}(t) ~\toa~ a
\end{array}
\]
To track that a rejection has been sent, the initiator reacts to the rejection
message by removing all proposals registered for the corresponding
agent and task:
\[
\onev~\ex{reject}(t) ~\toa~ a~\ifc~
\true
~\thendo~
\actname{remProps}(a,t)
\]
where $\actname{remProps}(\pname{a},\pname{t})$
:
\[
\big\{
\ex{PropPrice}(\pname{a},\pname{t},p)
\rightsquigarrow
\del\set{\ex{PropPrice}(\pname{a},\pname{t},p)
}
\big\}
\]
Finally, an assigned task is marked as $\str{done}$ whenever the
corresponding agent informs the initiator that the task has been
executed, or brought back to the $\str{todo}$ state if the agent
signals a failure. These two cases are respectively handled by the two
on-receive rules
\[
\begin{array}{@{}l@{}}
\onev~\ex{inform}(t) ~\froma~ a
\\\ifc~
\ex{AssignedTo}(a,t,\_)
~\thendo~
\actname{setState}(t,\str{done})
\end{array}
\]
and
\[
\begin{array}{@{}l@{}}
\onev~\ex{inform}(t) ~\froma~ a
\\
\ifc~
\ex{AssignedTo}(a,t,\_)
~\thendo~
\actname{setState}(t,\str{todo})
\end{array}
\]
which reuse the action $\actname{setState}$ as defined above. The case
of a failure allows the initiator agent to restart a contract net
protocol for the non-executed task.

%%% Local Variables:
%%% mode: latex
%%% TeX-master: "main"
%%% save-place: t
%%% End:

\section{Verification}
We now focus on the verification of RMASs against rich first-order
temporal properties.
%%
%\subsection{RMAS Execution Semantics}
\label{sec:ts}
% The execution semantics of an RMAS is defined as a \emph{relational transition
%   systems} that describes all possible agent executions,
% including the choices of results for service
% calls. The transition system is relational because its states contain database instances for $\inst$ and the other active agents.
% % The construction is a sensible variation of that for DACMASs,
% % where the most important difference relies on the management of data
% % types.
%%
% \begin{definition}
The execution semantics of RMAS $\X = \rmas$ is captured
by a \emph{relational transition system} $\ts_\X =
\tup{\ddom{\dtset},\schema_\X, \Sigma, \state_0,\db,\trans}$, where:
\begin{inparaenum}[\it (i)]
%\item $\ddom{\dtset}$ is the disjoint union of all data type domains.
\item $\schema_\X$ is the union of typed schemas in the
  specifications of $\agset$ and $\instspec$;
\item $\Sigma$ is a possibly infinite sets of \emph{states};
\item $\state_0 \in \Sigma$ is the \emph{initial state};
\item $\db$ is a function that, given a state $\state \in \Sigma$ and
  the name $\cv{n}$ of an agent active in $\state$, returns the database of $\cv{n}$ in state $\state$,
  written $\state.\db(\cv{n})$, which must be
  $\schema^{\cv{spec}_\cv{n}}$-conformant, where $\cv{spec}_\cv{n}$
  is the name of $\cv{n}$specification adopted by $\cv{n}$.
\item $\trans \subseteq \Sigma \times \Sigma$ is a transition relation
  between states.
\end{inparaenum}
%\end{definition}

\newcommand{\enaset}{\mathit{EMsg}}
\newcommand{\aset}{\mathit{ACT}}
\newcommand{\newag}{\mathit{NewAS}}
\newcommand{\curag}{\mathit{CurAS}}
\newcommand{\dbvar}{\mathit{DB}}
\newcommand{\comm}{\mathit{Commit}}
\newcommand{\callsvar}{\mathit{SCalls}}

\newcommand{\toadd}{\mathit{ToAdd}}
\newcommand{\toaddsk}{\mathit{ToAdd}_{s}}
\newcommand{\torem}{\mathit{ToDel}}
\algrenewcommand\algorithmicindent{1em}
\renewcommand{\gets}{:=}

\begin{figure*}[!t]
\begin{algorithmic}[1]
\small
\Procedure{build-ts}{$\sys$}%\Comment{The g.c.d. of a and b}
\\\textbf{input:} RMAS $\sys = \rmas$, \textbf{output:} Transition system  $\Upsilon_\X=\tup{\ddom{\dt},\Sigma, \state_0,\trans}$
\State $AS_{0} \gets \set{\tup{\cv{n},\cv{spec}_\cv{n}} \mid
  \hasSpec(\cv{n},\cv{spec}_\cv{n}) \in \idb^\inst}$ \Comment{Initial
  agents with their specifications}
\ForAll{$\tup{\cv{n},\cv{spec}_\cv{n}} \in AS_{0}$}
 $\state_0.\db(\cv{n}) \gets \idb^{\cv{spec}_\cv{n}}$
\Comment{Specify the initial state by extracting the initial DBs from
  the agent specs}
\EndFor
\State $\Sigma \gets \{\state_0\}$, $\trans \gets \emptyset$
%\State $AS_{old} \gets \set{\tup{\cv{a},\cv{spec}} \mid \hasSpec(\cv{a},\cv{s}) \in \state.\db(\inst)}$
\While{true}
\State \textbf{pick} $\state \in \Sigma$
\Comment{Nondepickterministically pick a state}
\State $\curag \gets \set{\tup{\cv{n},\cv{spec}_\cv{n}} \mid
  \hasSpec(\cv{a},\cv{spec}_\cv{n}) \in \state.\db(\inst)}$
\Comment{Get currently active agents with their specifications}
\State \textbf{pick} $\tup{\cv{a},\cv{spec}_\cv{a}} \in \curag$\Comment{Nondeterministically
  pick an active agent $\cv{a}$, elected as ``sender''}
%\State \textbf{pick} $\cv{a} \in \set{\cv{ag} \mid \agent(\cv{ag}) \in \state.\db(\inst)}$
%\Comment{Nondeterministically pick an agent in that state}

% \State $\enaset \gets \emptyset$ \Comment{Calculate the enabled
%   messages with target agents}
% \ForAll{communicative rules $``Q(t,\vec{x})
% ~\ena~ M(\vec{x}) ~\toa~ t"$ in $\cspec^{\cv{spec}_\cv{a}}$}
% \ForAll{$\subst \in \ans{Q,\state.\db(\cv{a})}$} \Comment{$\subst$
%   provides an actual payload and target agent}
% \If{$t\subst \in \set{\cv{n} \mid \tup{\cv{n},\_} \in \curag}$ and
%   $M(\vec{x})\subst$ conforms to $\typed{M} \in \evset$} 
% \Comment{$\subst$ is dynamically well-typed}
% \State $\enaset \gets \enaset \cup
% \tup{M(\vec{x},t)\subst,t\subst}$ \Comment{Add the ground event
%   and target agent to the set of enabled events}
% \EndIf
% \EndFor
% \EndFor

\State $\enaset \gets \Call{get-msgs}{\cspec^{\cv{spec}_\cv{a}},s.\db(\cv{a}),\curag}$ \Comment{Get the enabled
  messages with target agents}

\If{$\enaset \neq \emptyset$}
\State \textbf{pick} $\tup{M(\vec{\cv{o}}),\cv{b}} \in
\enaset$, with $\tup{\cv{b},\cv{spec}_\cv{b}} \in \curag$
\Comment{Pick a message+target agent and trigger message
  exchange and reactions}
\State $\aset_\cv{a} \gets \emptyset$, $\aset_\cv{b} \gets \emptyset$
\Comment{Get the actions with actual parameters to be applied by
  $\cv{a}$ and $\cv{b}$}
\ForAll{matching on-send rules $``\onev~ M(\vec{x})~\toa~ t~\ifc~ Q(t,\vec{x}) ~\thendo~
  \alpha(t,\vec{x})"$ in $\urule^{\cv{spec}_\cv{a}}$} 
\If{$\ans{Q(\cv{b},\vec{\cv{o}}},\state.\db(\cv{a}))$ and
  $\alpha(\cv{b},\vec{\cv{o}})$ conforms to $\typed{\alpha} \in \uact^\cv{a}$} 
$\aset_\cv{a} \gets \aset_\cv{a} \cup \alpha(\cv{b},\vec{\cv{o}})$
\label{alg:acta}
\EndIf
\EndFor
\ForAll{matching on-receive rules $``\onev~ M(\vec{x})~\froma~ s~\ifc~ Q(s,\vec{x}) ~\thendo~
  \alpha(s,\vec{x})"$  in $\urule^{\cv{spec}_\cv{b}}$} 
\If{$\ans{Q(\cv{a},\vec{\cv{o}}},\state.\db(\cv{b}))$ and
  $\alpha(\cv{a},\vec{\cv{o}})$ conforms to $\typed{\alpha} \in \uact^\cv{b}$} 
$\aset_\cv{b} \gets \aset_\cv{b} \cup \alpha(\cv{a},\vec{\cv{o}})$
\label{alg:actb}
\EndIf
\EndFor
% \If{$\cv{a} = \cv{b}$}
% \State $\tup{D'_{\cv{a}},Res_{\cv{a}}} \gets \tup{D'_{\cv{b}},Res_{\cv{b}}} \gets
% \Call{update}{\X,\state.\db(\cv{a}),\aset_{\cv{a}} \cup \aset_{\cv{b}}}$
% \Comment{Self-send: compute new DB for $\cv{a}$}
% \Else 
% \State $\tup{D'_{\cv{a}},Res_{\cv{a}}} \gets
% \Call{update}{\X,\state.\db(\cv{a}),\aset_{\cv{a}}}$; 
% $\tup{D'_{\cv{b}},Res_{\cv{b}}} \gets
% \Call{update}{\X,\state.\db(\cv{b}),\aset_{\cv{b}}}$
% \Comment{Compute new DBs }
% % \State $\tup{D'_{\cv{b}},Res_{\cv{b}}} \gets
% % \Call{update}{\X,\state.\db(\cv{b}),\aset_{\cv{b}}}$
% % \Comment{Compute new DB for target agent $\cv{b}$}
%\EndIf
\State $\tup{\torem^\cv{a},\toaddsk^\cv{a}} \gets
\Call{get-facts}{\sys,s.\db(\cv{a}),\aset_{\cv{a}}}$, 
 $\tup{\torem^\cv{b},\toaddsk^\cv{b}} \gets
\Call{get-facts}{\sys,s.\db(\cv{b}),\aset_{\cv{b}}}$
\State $\dbvar_s^{\cv{a}} \gets (s.\db(\cv{a}) \setminus
\torem^{\cv{a}}) \cup \toaddsk^{\cv{a}}$ \Comment{Calculate new
  $\cv{a}$'s DB, still with service calls to be issued}
\State $\dbvar_s^{\cv{b}} \gets (s.\db(\cv{b}) \setminus
\torem^{\cv{b}}) \cup \toaddsk^{\cv{b}}$ \Comment{Calculate new
  $\cv{b}$'s DB, still with service calls to be issued}
\If{for each $\scname{f}(\vec{\cv{o}}) \in \calls{\dbvar_s^{\cv{a}}
    \cup \dbvar_s^{\cv{b}}} $ with  $\typed{\scname{f}} = \tup{\scname{f}/n,\dfset^{in},\df^{out}}
  \in \scset$,  $\vec{\cv{o}}$ conforms to $\dfset^{in}$}
\Comment{Check service input types}
\label{alg:scin}
%\State $\sigma \gets \Call{get-call-res}{\sys,\state,\calls{\dbvar_s^{\cv{a}}
%    \cup \dbvar_s^{\cv{b}}}}$ \Comment{Evaluate service calls}
\State \textbf{pick }$ \sigma \in \left\{\theta \mid
\begin{array}{@{}l@{}}
\mathit{(i)\ } \theta \text{ is a total function, } 
\mathit{(ii)\ } \theta:\callsvar  \rightarrow
\ddom{\dtset},
\mathit{(iii)\ } \text{for each } \scname{f}(\vec{\cv{o}}), 
\scname{f}(\vec{\cv{o}})\theta \text{ conforms to } \df^{out}
\end{array}
\right\}
$% \Comment{Evaluate service calls}
\label{alg:scout}
\label{alg:sc}
% \ForAll{$\sigma \in \Theta$} \Comment{Consider each possible scenario,
% i.e., each possible combination of results from service calls}

% \If{
% $\left(
% \begin{array}{@{~}r@{~}l@{~}}
% D'_{\cv{a}} \text{ conforms to } \schema^\cv{a} & \text{ and }
% \constr^{\cv{a}} 
% \text{ is satisfied by } D'_{\cv{a}} \text{ and } \\
% D'_{\cv{b}} \text{ conforms to } \schema^\cv{b} & \text{ and }
% \constr^{\cv{b}} 
% \text{ is satisfied by } D'_{\cv{b}} \\%\text{ and }\\
% %ACT_\cv{a} \neq \emptyset & \text{ and } ACT_\cv{b} \neq \emptyset 
% \end{array}
% \right)$
% }
% \Comment{Commit update only if $\cv{a}$ and $\cv{b}$ produce acceptable new DBs}
% \State $\dbvar_{\cv{a}} \gets (s.\db(\cv{a}) \setminus
% \torem^{\cv{a}}) \cup \toaddsk^{\cv{a}}\sigma$ \Comment{Ground service
%   calls
%   and update
%   $\cv{a}$'s DB, giving higher priority to
% additions}

% \State $\dbvar_{\cv{b}} \gets (s.\db(\cv{b}) \setminus
% \torem^{\cv{b}}) \cup \toaddsk^{\cv{b}}\sigma$ \Comment{Ground service
%   calls
%   and update
%   $\cv{b}$'s DB, giving higher priority to
% additions}
\State $\dbvar^{\cv{a}}_{\mathit{cand}} \gets \dbvar_s^{\cv{a}}\sigma, \dbvar^{\cv{b}}_{\mathit{cand}}
\gets \dbvar_s^{\cv{b}}\sigma$
\Comment{Obtain new candidate DBs by substituting service calls with
  results}
\If{$\dbvar^{\cv{a}}_{\mathit{cand}}$ conforms to
  $\schema^\cv{a}$) $\land$ ($\dbvar^{\cv{a}}_{\mathit{cand}}$ satisfies
  $\constr^{\cv{a}}$)}
$\dbvar_{\cv{a}} \gets \dbvar^{\cv{a}}_{\mathit{cand}}$
\Comment{Update $\cv{a}$'s DB}
\label{alg:commita}
\Else\ 
$\dbvar_{\cv{a}} \gets \state.\db(\cv{a})$
\Comment{Rollback $\cv{a}$'s DB}
\EndIf
\If{$\dbvar^{\cv{b}}_{\mathit{cand}}$ conforms to
  $\schema^\cv{b}$) $\land$ ($\dbvar^{\cv{b}}_{\mathit{cand}}$ satisfies
  $\constr^{\cv{b}}$)}
$\dbvar_{\cv{b}} \gets \dbvar^{\cv{b}}_{\mathit{cand}}$
\Comment{Update $\cv{b}$'s DB}
\label{alg:commitb}
\Else\ 
$\dbvar_{\cv{b}} \gets \state.\db(\cv{b})$
\Comment{Rollback $\cv{b}$'s DB}
\EndIf
% \If{$\comm_{\cv{a}} \lor \comm_{\cv{b}}$} \Comment{Proceed only if at
%   least one of the involved agent needs to update its data}
\State \textbf{pick} fresh state $\state'$ \Comment{Create new state}
\State $\newag \gets \emptyset$ \Comment{Determine the (possibly
  changed)   set of active agents and their specs}
\If{$\cv{a} = \cv{\inst}$}
 $\newag \gets \set{\tup{\cv{n},\cv{spec}_\cv{n}} \mid \hasSpec(\cv{n},\cv{spec}_\cv{n}) \in \dbvar_{\cv{a}}}$
\ElsIf{$\cv{b} = \cv{\inst}$}
 $\newag \gets \set{\tup{\cv{n},\cv{spec}_\cv{n}} \mid \hasSpec(\cv{n},\cv{spec}_\cv{n}) \in \dbvar_{\cv{b}}}$
\Else\ $\newag \gets \curag$
\Comment{No change if $\inst$ is not involved in the interaction or
  must reject the update}
%\gets \set{\tup{\cv{n},\cv{spec}_\cv{n}} \mid \tup{\cv{n},\cv{spec}_\cv{n}} \in AS_{cur}}$
\EndIf
\ForAll{$\tup{\cv{n},\cv{spec}_\cv{n}} \in \newag$} \Comment{Do the
  update for each active agent}
\If{$\cv{n} = \cv{a}$}
  $\state'.\db(\cv{n}) \gets \dbvar_{\cv{a}}$ 
  \Comment{Case of sender agent} 
\ElsIf{$\cv{n} = \cv{b}$} 
  $\state'.\db(\cv{n}) \gets \dbvar_{\cv{b}}$
\Comment{Case of target agent}
\ElsIf{$\cv{n} \not\in \curag$} 
\Comment{Case of newly created agent}
\State $\state'.\db(\cv{n}) \gets
\idb^{\cv{spec}_\cv{n}} \cup \set{\myname(\cv{n})}$
\Comment{$\cv{n}$'s initial DB gets the initial data fixed by its
  specification, plus its name}
\Else\ $\state'.\db(\cv{n}) \gets \state.\db(\cv{n})$
\Comment{Default case: persisting agent not affected by the interaction}
\EndIf
\EndFor
\If{ $\exists \state'' \in \Sigma$ s.t.~$\state''.\db(\inst) =
  \state'.\db(\inst)$ and for each $\tup{\cv{n},\_} \in \curag$,
  $\state''.\db(\cv{n}) = \state'.\db(\cv{n})$ }
\State $\trans \gets \trans \cup \tup{\state,\state''}$
\Comment{State already exists: connect $\state$ to that state}
\Else
\ $\Sigma \gets \Sigma \cup \set{\state'}$, 
$\trans \gets \trans \cup \tup{\state,\state'}$
\Comment{Add and connect new state}
\EndIf
% \Else \Comment{Both $\cv{a}$ and $\cv{b}$ fail to update their data}
% \State $\trans \gets \trans \cup \tup{\state,\state}$
% \Comment{No DB is changed}
%  \EndIf
%\EndFor
\EndIf
\EndIf%\EndIf

\EndWhile
%\State \textbf{return} $\tup{\const,\tbox \cup \cbox,\Sigma,
%  s_0,\db,\trans}$
\EndProcedure

\Function{get-msgs}{$\cspec$,$\dbvar$,$\curag$}
\Comment{Evaluate communicative rules $\cspec$ on DB $\dbvar$, and
  return the enabled messages with targets}
\State $\enaset \gets \emptyset$
\ForAll{communicative rules $``Q(t,\vec{x})
~\ena~ M(\vec{x}) ~\toa~ t"$ in $\cspec$}
\ForAll{$\subst \in \ans{Q,\dbvar}$} \Comment{$\subst$
  provides an actual payload and target agent}
\label{alg:msg-init}
\If{$t\subst \in \set{\cv{n} \mid \tup{\cv{n},\_} \in \curag}$ and
  $M(\vec{x})\subst$ conforms to $\typed{M} \in \evset$} 
\Comment{$\subst$ is well-typed and has an active agent as target}
\State $\enaset \gets \enaset \cup
\tup{M(\vec{x},t)\subst,t\subst}$ \Comment{Add the ground event
  and target agent to the set of enabled events}
\label{alg:msg-end}
\EndIf 
\EndFor
\EndFor
\State \textbf{return} $\enaset$
\EndFunction

\Function{get-facts}{$\X$, $\dbvar$, $\aset$}
\Comment{Applies actions $\aset$ on DB $\dbvar$, and returns 
  facts to be added and deleted}
\State $\toaddsk \gets \emptyset$; $\torem \gets \emptyset$ 
\Comment{$\toaddsk$: facts with
  embedded service calls, to be added; $\torem$: facts to be deleted}
\ForAll{instantiated actions $\alpha(\vec{\cv{v}}) \in \aset$}
\ForAll{effects $``\map{Q(\vec{p},\vec{x})}{\add~A,~\del~D}"$
  in the definition of $\alpha$}
\ForAll{$\theta \in \ans{Q(\vec{\cv{v}},\vec{x}), D}$} \Comment{Get an
    answer from the left-hand side}
\State $\toaddsk \gets \toaddsk \cup A\theta[\vec{p}/\vec{\cv{v}}]$
\Comment{Get facts to add (with embedded service calls)} 
\State $\torem \gets \torem \cup D\theta[\vec{p}/\vec{\cv{v}}]$
\Comment{Get facts to delete}
\EndFor
\EndFor
\EndFor
\State \textbf{return} $\tup{\torem,\toaddsk}$
\Comment{Recall: facts to be added still contain service calls - to be
substituted with actual results}
\EndFunction

% \Function{get-call-res}{$\sys$,$\state$,$\callsvar$} \Comment{Returns a
%   combination of results for service calls $\callsvar$ issued in state
% $\state$}
% \State \textbf{pick }$ \sigma \in \left\{\theta \mid
% \begin{array}{@{}l@{}}
% \mathit{(i)\ } \theta \text{ is a total function, } 
% \mathit{(ii)\ } \theta:\callsvar  \rightarrow
% \ddom{\dtset},
% \mathit{(iii)\ } \text{for each } \scname{f}(\vec{\cv{o}}), 
% \scname{f}(\vec{\cv{o}})\theta \text{ conforms to } \df^{out}
% \end{array}
% \right\}
% $
% \State \textbf{return} $\sigma$
% \EndFunction
\end{algorithmic}
\vspace{-.3cm}
\caption{Procedure for constructing a transition system describing the
execution semantics of an RMAS; given a set $F$ of facts, $\calls{F}$
returns the ground service calls contained in $F$}\label{alg:ts}
\end{figure*}

%%% Local Variables:
%%% mode: latex
%%% TeX-master: "main"
%%% save-place: t
%%% End:

The full $\ts_\X$ construction starting from the initial state is
  given in Figure~\ref{alg:ts}. We report the main steps in the following. The initial state
$\state_0$ is constructed by assigning $\state_0.\db(\inst)$ to the
initial database instance $\idb^\cv{ispec}$ of $\instspec$, and the initial
database of each agent mentioned in  $\idb^\cv{ispec}$ taking from
its specification. The construction then proceeds by mutual induction
over $\Sigma$ and $\trans$, repeating the following steps forever:
\begin{inparaenum}[(1)]
\item A state $\state$ is picked from $\Sigma$.
\item An active agent $\cv{a}$ is nondeterministically picked
  selecting its name from $\state.\db(\inst)$.
\item The communicative rules of $\cv{a}$ are evaluated, extracting
  all enabled messages with their ground payloads and destination agents.
\item An enabled messages is nondeterministically picked.
\item The on-send/on-receive rules of the two involved agents are
  triggered, fetching all actions to be applied.
\item The actions are applied over the respective databases. If there
  are service calls involved, they are nondeterminstically substituted
  with resulting data objects, consistently with the service output facets.
\item Each agent updates its own database provided that the database
  resulting from the parallel application of the actions is compatible
  with the schema and satisfies all constraints. Otherwise the old
  database is maintained, so as to model a sort of ``transaction rollback''.
\item If one of the involved agents is $\inst$ and the update leads to
  the introduction of a new agent into the system, it database is
  initialized in accordance to its specification.
\item The global state so obtained is declared to be successor of the
  state picked at step 1.
\end{inparaenum}

Interestingly, $\ts_\X$ is in general \emph{infinite-branching}, because of
the substitution of service calls with their results, and \emph{infinite
runs}, because of the storage of such data objects in time.

%\subsection{The $\mulpd$ Verification Logic}
\smallskip
\noindent \textbf{The $\mulpd$ Verification Logic.}
To specify sophisticated properties over RMASs we employ the
$\mulpd$ logic. This logic combines the salient features of those
introduced in \cite{BCDDM13} and \cite{MoCD14}.
%with the addition of typed relations.
$\mulpd$ supports the full $\mu$-calculus to predicate over the system
  dynamics. Recall that the $\mu$-calculus is virtually the most
  expressive temporal logics: it subsumes LTL and CTL$^*$. To query
  possibly different agent databases, $\mulpd$ adopts FO queries
  extended with location arguments \cite{MoCD14}, which are dynamically bound to agents.
Furthermore, to track the temporal evolution of data objects, $\mulpd$
adopts a controlled form of FO quantification across time: quantification is limited to those objects
  that \emph{persist} in the system:
% (i.e., are present in the active
  %domain of some active agent).
%
% The last aspect is a tight requirement for the logic, as it
% cannot be relaxed without compromising decidability under the
% assumptions that we will exploit in Section~\ref{sec:dec}.\footnote{Cf.~\cite{BCDD13}, for which undecidability is shown for DCDSs, which
% intuitively correspond to RMASs with a single agent and a single data
% type equipped with equality only.} Notably, it reflects the nature of
% ``pure'' names (such as agent names), for which the identity of a name
% carries over only until the name is actively present in the system
% \cite{???}.
%
%Technically,
%$\mulpd$ is defined as:
\[
\small
\begin{array}{@{}l@{}}
  \Phi ::=\  Q_\ell \mid \lnot \Phi \mid \Phi_1 \land \Phi_2 \mid \exists
    x. \inadom_\dt(x) \land \Phi \mid Z \mid \mu Z.\Phi \mid{}\\
  \bigwedge_{i \in \set{1,\ldots,n}} \inadom_{\dt_i}(\vec{x_i}) \land
  \DIAM{\Phi}  \mid \bigwedge_{i \in \set{1,\ldots,n}} \inadom_{\dt_i}(\vec{x_i}) \land \BOX{\Phi}  \\
 \end{array}
\]

where $Q_\ell$ is a (possibly open) FO query with location arguments, in which the only
constants that may appear are those in $\cdom{\dfset}$, and $Z$ is a second order
predicate variable (of arity 0).  Furthermore, the following assumption holds:
in the $\DIAM{}$ and $\BOX{}$ cases, the variables $x_1,\ldots,x_n$
are exactly the free variables of $\Phi$, once we substitute to each bounded predicate variable $Z$ in $\Phi$ its bounding formula $\mu Z.\Phi'$.
We  adopt the usual abbreviations, including $\nu Z. \Phi$ for
greatest fixpoints.
Notice that the usage of $\inadom$ can be safely substituted by an
atomic positive query.

The semantics of $\mulpd$ is defined over a relational transition
system similarly to the semantics of $\mu\L_p$ in
\cite{BCDDM13}. The most peculiar aspect is constituted by $Q_\ell$,
which allows one to dynamically inspect the
databases maintained by active agents. In particular, $Q_\ell$ is a
standard (typed) FO query, whose atoms have the form $R(\vec{x})@a$,
where $R$ is a (typed) relation, and $a$ denotes an agent name. The
evaluation of the atomic query $R(\vec{x})@a$ over a relational
transition system $\ts$ with substitution $\theta$ returns those
states $\state$ of $\ts$ such that:
\begin{compactitem}
\item $a\theta$ is an active agent  in $\state$, that is,
  $\agent(a\theta) \in \state.\db(\inst)$;
\item the atomic query $R(\vec{x})\theta$ evaluates to $\true$ in the
  database instance that agent $a\theta$ has in state $\state$, i.e.,
$\ans{R(\vec{x})\theta,\state.\db(a\theta)} \equiv \true$.
\end{compactitem}

\begin{example}
Consider the protocol in Section~\ref{sec:example}, assuming that $\inst$ uses a unary typed
relation $\ex{inCritical}$ to store the agent that is currently
in the critical interaction. Given:
\[\small
\begin{array}[t]{@{}l@{}}
\ex{First}(a) = \exists t. \ex{hasTicket}@\inst(a,t) \land \\\neg
\exists a',t'. \ex{hasTicket}@\inst(a',t') \land a' \neq a \land t' < t,\\[4pt]
% \end{array}\]
% % The following $\mulpd$ property models
% % that an agent can be involved in the critical interaction only if it
% % is currently first:
% % \[
% % \small
% % \begin{array}{@{}l@{}}
% % \nu Z. (\forall a. \agent@\inst(a) \land \ex{inCritical}@\inst(a)
% % \limp \ex{First}(a))\land \BOX{Z}
% % \end{array}
% % \]
% %The following $\mulpd$ property:
% \[
% \small
%\begin{array}{@{}l@{}}
\nu Z. (\forall a. \agent@\inst(a) \land \ex{First}(a) \limp \\
\qquad\mu Y. (\ex{inCritical}@\inst(a) \lor (\agent@\inst(a) \land \DIAM{Y}))
\land \BOX{Z}
\end{array}
\]
models that when an
agent is ``first'', there will be a run in which it
persists into the system until it enters the critical
interaction.
\end{example}

\section{Decidability of Verification}
\label{sec:dec}
We now study different aspects of the following \emph{verification problem}: given a closed $\mulpd$
property $\Phi$ and an RMAS $\sys$, check whether $\Phi$ holds over the
relational transition system $\ts_\sys$, written $\ts_\sys
\models \Phi$.
Unsurprisingly, this problem in general is undecidable.
% , even in presence of a single agent that employs
% unary relations only.\footnote{The proof is a trivial consequence of
%   \cite{BCDDM13,MoCD14}.}
%
In a recent series of works, verification of data-aware dynamic systems has
been studied under the notion of \emph{state-boundedness} \cite{BCDM14}, which,
in the context of RMASs, can be phrased as follows.
%%
%\begin{definition}
An RMAS $\sys$ is \emph{state-bounded} if, for every state $\state$ of
$\ts_\sys$, the number of data objects stored in each agent database
does not exceed a pre-defined bound.
%\end{definition}

As shown in previous work, state-boundedness still allows one to model systems
that encounter infinitely many different data objects (and, possibly, even
agents) along their runs, provided that they do not accumulate in the same
state. In our setting, this means that infinitely many different agents can
interact, provided that at each time point only a bounded number of them is
active \cite{MoCD14}. Similarly, from Theorem~\ref{thm:asynch} we obtain that
when an RMAS is state-bounded, asynchronous communication can be modelled only
by putting a threshold on the size of each message queue.

\cite{MoCD14} have shown that verification of state-bounded DACMASs is
decidable. We study now how data types impact on this.

\smallskip\noindent
\textbf{Compilation of Facets.~}
%\label{sec:compilation}
%
Facets can be % safely
eliminated, getting
%\begin{definition}
a \emph{shallow-typed} RMAS, i.e., one using base facets only.
%\end{definition}

%With this class at hand, we can prove the following result:

\begin{theorem}
\label{thm:facet}
For every RMAS $\sys$, there exists a corresponding shallow-typed RMAS
$\widehat{\sys}$ such that, for every $\mulpd$ property $\Phi$, we
have $\ts_\sys \models \Phi$ if and only if $\ts_{\widehat{\sys}}
\models \Phi$.
\end{theorem}

\begin{proof}
Let $\sys = \rmas$. We construct $\widehat{\sys} =  
\tup{\dtset, \widehat{\dtset}, \cdom{\dfset},\widehat{\scset}, \widehat{\evset},
    \widehat{\agset},\widehat{\instspec}}$ as follows:
\begin{compactitem}
\item $\widehat{\dtset}$ is the set of base facets constructed
  starting from the types in $\dtset$.
\item $\widehat{\scset}$ and $\widehat{\evset}$,
    are obtained from
    $\scset$ and $\evset$ by substituting the facet attached to each 
    component with the corresponding base facet: whenever a
    component is originally typed with facet $\tup{\dt,\varphi(x)} \in
    \dfset$, the corresponding component is typed with the base facet
    $\tup{\dt,\true} \in \widehat{\dtset}$.
\item $\widehat{\scset}$ and $\widehat{\evset}$,
    are obtained from
    $\scset$ and $\evset$ by substituting the facet attached to each 
    component with the corresponding base facet: whenever a
    component is originally typed with facet $\tup{\dt,\varphi(x)} \in
    \dfset$, the corresponding component is typed with the base facet
    $\tup{\dt,\true} \in \widehat{\dtset}$.
  \item Each agent specification $\N = \tup{\name,\schema,\constr,\idb,\crule,\uact,\urule}$ in $\agset \cup
\set{\instspec}$ becomes a corresponding agent specification $\widehat{\N} = \tup{\name,\widehat{\schema},\widehat{\constr},\idb,\widehat{\crule},\widehat{\uact},\widehat{\urule}}$ in $\widehat{\agset} \cup
\set{\widehat{\instspec}}$. The database schema $\widehat{\schema}$
transforms $\schema$ similarly to how $\widehat{\scset}$ and
$\widehat{\evset}$ transform ${\scset}$ and ${\evset}$: for every
$n$-ary typed relation $\typed{R} \in \schema$, a corresponding
$n$-ary relation $\typed{R}$ is included in $\widehat{\schema}$, such
that, for every $i \in \set{1,\ldots,n}$, $\type{R'}{i}{\widehat{\schema}} =
\tup{T,\true}$ if and only if $\type{R}{i}{\schema} =
\tup{T,\varphi(x)}$. In addition, for every typed service call
$\typed{f}(\tup{T_1,\varphi_1(x)},\ldots,\tup{T_n,\varphi_n(x)})$ in $\scset$,
$\widehat{\schema}$ contains a relation
$\typed{\mathit{Input}_f}(\tup{T_1,\varphi_1(x)},\ldots,\tup{T_n,\varphi_n(x)})$,
whose use is explained below.

The other elements of $\widehat{\N}$ ensure that the type checks of $\N$ are properly recreated in
the form of special queries and constraints. In particular:
\begin{compactitem}
\item For every communicative rule
``$Q(t,\vec{x})$ $\ena$ $M(\vec{x})$ $\toa$ $t$''
in $\crule$, with $|\vec{x}| = n$, $\widehat{\crule}$ contains the corresponding rule
\[
Q(t,\vec{x}) \land \hspace*{-1cm}\bigwedge_{i \in \set{1,\ldots,n},
  \tup{T_i,\varphi_i(x)} = \type{M}{i}{\evset}} \hspace*{-.5cm}
\varphi(x_i) 
~\ena~M(\vec{x})
\]
This guarantees that the filter criterion applied on lines \ref{alg:msg-init}-\ref{alg:msg-end} of
Figure~\ref{alg:ts} is properly reconstructed, so that $\sys$ and
$\widehat{\sys}$ produce the same sets of enabled messages.
\item A similar approach is applied to the update rules in $\urule$,
  incorporating into each condition the facet expressions of the facets
  attached to the corresponding action components, in such a way that the filter
  criterion applied on lines \ref{alg:acta} and \ref{alg:actb} of Figure~\ref{alg:ts} is
  properly reconstructed. This ensures that $\sys$ and
$\widehat{\sys}$ produce the same sets of instantiated actions.
\item Actions $\uact$ need to be translated by ensuring that the types
  of relations in $\schema$ and those of the service call input/outputs in
  $\scset$ are properly checked. The typing of relation components is
  guaranteed by augmenting the set $\constr$ of
  constraints. Specifically, beside
  all the original  constraints in $\constr$,
for each $n$-ary typed relation $\typed{R}$
  in $\schema$ and every $i \in \set{1,\ldots,n}$, we insert into
  $\widehat{\constr}$ a dedicated constraint
\[
\forall
  x_i. R(\_,\ldots,x_i,\ldots,\_) \limp \varphi_i(x_i)
\]
where $\varphi_i$ is the facet formula of $\type{R}{i}{\schema}$. This
technique guarantees that $\sys$ and $\widehat{\sys}$ equivalently
evaluate the conditions on lines \ref{alg:commita} and  \ref{alg:commitb} of Figure~\ref{alg:ts}
($\widehat{\sys}$ always satisfies the conformance test, and lifts
the original conformance test of $\sys$ as a test on the
satisfaction of database constraints, expressed in the second conjunct
of lines  \ref{alg:commita} and  \ref{alg:commitb}). Finally, the tests expressed
on lines \ref{alg:scin} and \ref{alg:scout} of Figure~\ref{alg:ts},
which respectively check whether the service calls involved in an
action application have inputs and outpus conforming to their
respective facets, is reformulated using the technique illustrated in the following. For every action
$\typed{\alpha} \in \uact$, $\widehat{\uact}$ contains an action
$\typed{\alpha}'$, constructed by properly manipulating the set of
facts in the $\add$-set. Specifically, for each effect
``$\map{Q(\vec{\pname{p}},\vec{x})}{\add~A,\del~D} $'' in the
specification of $\alpha$, $\alpha'$ contains a corresponding effect
``$\map{Q(\vec{\pname{p}},\vec{x})}{\add~A',\del~D}$'', where:
\[
\small
\begin{array}{@{}l@{}l@{}}
A' = A \cup 
& \set{\ex{Input}_f(\vec{x}) | F \in A \text{ and }
  \scname{f}(\vec{x}) \text{ appears in } F}\\
& \set{\ex{Output}_f(\scname{f}(\vec{x})) | F \in A \text{ and }
  \scname{f}(\vec{x}) \text{ appears in } F}
\end{array}
\]
Intuitively, $A'$ adds a fact for relation $\ex{Input}_f/n$ and a fact
for relation $\ex{Output}_f/1$ for every $n$-ary service call
$\scname{f}$ appearing in $A$, in such a way that the contect of these two facts
respectively correspond to the input and output of $\scname{f}$. Since it is not important that such
facts are persisted in the agent database, but it is only important
that they are present after the action is applied, the specification
of each action in $\widehat{\uact}$ also contains the following
effects:
\[
\big\{\ex{Input}_{f_i}(\vec{x}) \rightsquigarrow
  \del\set{\ex{Input}_{f_i}(\vec{x})}
~\big|~ \typed{f_i} \in \scset
\big\}
\]
The conformance with the service input facets can then be reformulated
similarly to the case of relations in $\schema$, that is, by further
augmenting the set $\widehat{\constr}$ of constraints. Specifically,
for each $n$-ary service call $\typed{\scname{f}} = \tup{\scname{f}/n,\dfset^{in},\df^{out}}$
  in $\scset$, we insert two
  dedicated constraints in
  $\widehat{\constr}$:
\begin{compactenum}
\item by denoting with $\varphi_i$ the facet formula of the $i$-th component of $\dfset^{in}$,
\[\forall
  x_1,\ldots,x_n. \ex{Input}_f(x_1,\ldots,x_n) \limp \bigwedge_{i \in
    \set{1,\ldots,n}}\varphi_i(x_i)\]
\item by denotwing with $\psi$ the facet formula of $\df^{out}$,
\[\forall
  x. \ex{Output}_f(x) \limp \psi(x)\]
\end{compactenum} This
mechanism lifts the checks applied for $\sys$ on lines \ref{alg:scin}
and \ref{alg:scout} of Figure~\ref{alg:ts}
(which is trivially true for $\widehat{\sys}$) as additional
constraint checks on lines \ref{alg:commita} and \ref{alg:commitb}, where the satisfaction of
database constraints is tested. 
\end{compactitem}
\end{compactitem}
The translation mechanism ensures that the execution semantics of
$\widehat{\sys}$ suitably reconstructs that of $\sys$, i.e., if we
project away the accessory relations used for the service call inputs,
we have that
$\ts_{\widehat{\sys}}$ is equivalent to $\ts_\sys$. 
\end{proof}
As a consequence of Theorem~\ref{thm:facet}, we have that, for
shallow-typed RMASs, the transition system construction can be
simplified as shown in the \textsc{build-ts-shallow} procedure of
Figure~\ref{alg:ts-sh}.

\begin{figure*}[!t]
\begin{algorithmic}[1]
\small
\Procedure{build-ts-shallow}{$\widehat{\sys}$}%\Comment{The g.c.d. of a and b}
\\\textbf{input:} Shallow-typed RMAS $\widehat{\sys}  = \tup{\dtset, \widehat{\dtset},
  \cdom{\dfset},\widehat{\scset}, \widehat{\evset}}$, \textbf{output:} Transition system  $\Upsilon_\X=\tup{\ddom{\dt},\Sigma, \state_0,\trans}$
\State $AS_{0} \gets \set{\tup{\cv{n},\cv{spec}_\cv{n}} \mid
  \hasSpec(\cv{n},\cv{spec}_\cv{n}) \in \idb^\inst}$ \Comment{Initial
  agents with their specifications}
\ForAll{$\tup{\cv{n},\cv{spec}_\cv{n}} \in AS_{0}$}
 $\state_0.\db(\cv{n}) \gets \idb^{\cv{spec}_\cv{n}}$
\Comment{Specify the initial state by extracting the initial DBs from
  the agent specs}
\EndFor
\State $\Sigma \gets \{\state_0\}$, $\trans \gets \emptyset$
%\State $AS_{old} \gets \set{\tup{\cv{a},\cv{spec}} \mid \hasSpec(\cv{a},\cv{s}) \in \state.\db(\inst)}$
\While{true}
\State \textbf{pick} $\state \in \Sigma$
\Comment{Nondeterministically pick a state}
\State $\curag \gets \set{\tup{\cv{n},\cv{spec}_\cv{n}} \mid
  \hasSpec(\cv{a},\cv{spec}_\cv{n}) \in \state.\db(\inst)}$
\Comment{Get currently active agents with their specifications}
\State \textbf{pick} $\tup{\cv{a},\cv{spec}_\cv{a}} \in \curag$\Comment{Nondeterministically
  pick an active agent $\cv{a}$, elected as ``sender''}
%\State \textbf{pick} $\cv{a} \in \set{\cv{ag} \mid \agent(\cv{ag}) \in \state.\db(\inst)}$
%\Comment{Nondeterministically pick an agent in that state}

% \State $\enaset \gets \emptyset$ \Comment{Calculate the enabled
%   messages with target agents}
% \ForAll{communicative rules $``Q(t,\vec{x})
% ~\ena~ M(\vec{x}) ~\toa~ t"$ in $\cspec^{\cv{spec}_\cv{a}}$}
% \ForAll{$\subst \in \ans{Q,\state.\db(\cv{a})}$} \Comment{$\subst$
%   provides an actual payload and target agent}
% \If{$t\subst \in \set{\cv{n} \mid \tup{\cv{n},\_} \in \curag}$ and
%   $M(\vec{x})\subst$ conforms to $\typed{M} \in \evset$} 
% \Comment{$\subst$ is dynamically well-typed}
% \State $\enaset \gets \enaset \cup
% \tup{M(\vec{x},t)\subst,t\subst}$ \Comment{Add the ground event
%   and target agent to the set of enabled events}
% \EndIf
% \EndFor
% \EndFor

\State $\enaset \gets \Call{get-msgs}{\widehat{\cspec}^{\cv{spec}_\cv{a}},s.\db(\cv{a}),\curag}$ \Comment{Get the enabled messages with target agents}

\If{$\enaset \neq \emptyset$}
\State \textbf{pick} $\tup{M(\vec{\cv{o}}),\cv{b}} \in
\enaset$, with $\tup{\cv{b},\cv{spec}_\cv{b}} \in \curag$
\Comment{Pick a message+target agent and trigger message
  exchange and reactions}
\State $\aset_\cv{a} \gets \emptyset$, $\aset_\cv{b} \gets \emptyset$
\Comment{Get the actions with actual parameters to be applied by
  $\cv{a}$ and $\cv{b}$}
\ForAll{matching on-send rules $``\onev~ M(\vec{x})~\toa~ t~\ifc~ Q(t,\vec{x}) ~\thendo~
  \alpha(t,\vec{x})"$ in $\widehat{\urule}^{\cv{spec}_\cv{a}}$} 
\If{$\ans{Q(\cv{b},\vec{\cv{o}}},\state.\db(\cv{a}))$} 
$\aset_\cv{a} \gets \aset_\cv{a} \cup \alpha(\cv{b},\vec{\cv{o}})$
\label{alg:acta-sh}
\EndIf
\EndFor
\ForAll{matching on-receive rules $``\onev~ M(\vec{x})~\froma~ s~\ifc~ Q(s,\vec{x}) ~\thendo~
  \alpha(s,\vec{x})"$  in $\widehat{\urule}^{\cv{spec}_\cv{b}}$} 
\If{$\ans{Q(\cv{a},\vec{\cv{o}}},\state.\db(\cv{b}))$} 
$\aset_\cv{b} \gets \aset_\cv{b} \cup \alpha(\cv{a},\vec{\cv{o}})$
\label{alg:actb-sh}
\EndIf
\EndFor
\State $\tup{\torem^\cv{a},\toaddsk^\cv{a}} \gets
\Call{get-facts}{\widehat{\sys},s.\db(\cv{a}),\aset_{\cv{a}}}$, 
 $\tup{\torem^\cv{b},\toaddsk^\cv{b}} \gets
\Call{get-facts}{\widehat{\sys},s.\db(\cv{b}),\aset_{\cv{b}}}$
\State $\dbvar_s^{\cv{a}} \gets (s.\db(\cv{a}) \setminus
\torem^{\cv{a}}) \cup \toaddsk^{\cv{a}}$ \Comment{Calculate new
  $\cv{a}$'s DB, still with service calls to be issued}
\State $\dbvar_s^{\cv{b}} \gets (s.\db(\cv{b}) \setminus
\torem^{\cv{b}}) \cup \toaddsk^{\cv{b}}$ \Comment{Calculate new
  $\cv{b}$'s DB, still with service calls to be issued}
% \State $\sigma \gets \Call{get-call-results}{\widehat{\sys},\state,\calls{\dbvar_s^{\cv{a}}
%     \cup \dbvar_s^{\cv{b}}}}$ \Comment{evaluate service
%     calls}\label{alg:scout-sh}
\State \textbf{pick }$ \sigma \in \left\{\theta \mid
\begin{array}{@{}l@{}}
\mathit{(i)\ } \theta \text{ is a total function, } 
\mathit{(ii)\ } \theta:\callsvar  \rightarrow
\ddom{\dtset},
\mathit{(iii)\ } \text{for each } \scname{f}(\vec{\cv{o}}), 
\scname{f}(\vec{\cv{o}})\theta \text{ conforms to } \df^{out}
\end{array}
\right\}
$% \Comment{Evaluate service calls}
\label{alg:sc-sh}
\State $\dbvar^{\cv{a}}_{\mathit{cand}} \gets \dbvar_s^{\cv{a}}\sigma, \dbvar^{\cv{b}}_{\mathit{cand}}
\gets \dbvar_s^{\cv{b}}\sigma$
\Comment{Obtain new candidate DBs by substituting service calls with
  results}
\If{$\dbvar^{\cv{a}}_{\mathit{cand}}$ satisfies
  $\widehat{\constr}^{\cv{a}}$}
$\dbvar_{\cv{a}} \gets \dbvar^{\cv{a}}_{\mathit{cand}}$
\Comment{Update $\cv{a}$'s DB}
\label{alg:commita-sh}
\Else\ 
$\dbvar_{\cv{a}} \gets \state.\db(\cv{a})$
\Comment{Rollback $\cv{a}$'s DB}
\EndIf
\If{$\dbvar^{\cv{b}}_{\mathit{cand}}$ satisfies
  $\widehat{\constr}^{\cv{b}}$}
$\dbvar_{\cv{b}} \gets \dbvar^{\cv{b}}_{\mathit{cand}}$
\Comment{Update $\cv{b}$'s DB}
\label{alg:commitb-sh}
\Else\ 
$\dbvar_{\cv{b}} \gets \state.\db(\cv{b})$
\Comment{Rollback $\cv{b}$'s DB}
\EndIf
\State \textbf{pick} fresh state $\state'$ \Comment{Create new state}
\State $\newag \gets \emptyset$ \Comment{Determine the (possibly
  changed)   set of active agents and their specs}
\If{$\cv{a} = \cv{\inst}$}
 $\newag \gets \set{\tup{\cv{n},\cv{spec}_\cv{n}} \mid \hasSpec(\cv{n},\cv{spec}_\cv{n}) \in \dbvar_{\cv{a}}}$
\ElsIf{$\cv{b} = \cv{\inst}$}
 $\newag \gets \set{\tup{\cv{n},\cv{spec}_\cv{n}} \mid \hasSpec(\cv{n},\cv{spec}_\cv{n}) \in \dbvar_{\cv{b}}}$
\Else\ $\newag \gets \curag$
\Comment{No change if $\inst$ is not involved in the interaction or
  must reject the update}
\EndIf
\ForAll{$\tup{\cv{n},\cv{spec}_\cv{n}} \in \newag$} \Comment{Do the
  update for each active agent}
\If{$\cv{n} = \cv{a}$}
  $\state'.\db(\cv{n}) \gets \dbvar_{\cv{a}}$ 
  \Comment{Case of sender agent} 
\ElsIf{$\cv{n} = \cv{b}$} 
  $\state'.\db(\cv{n}) \gets \dbvar_{\cv{b}}$
\Comment{Case of target agent}
\ElsIf{$\cv{n} \not\in \curag$} 
\Comment{Case of newly created agent}
\State $\state'.\db(\cv{n}) \gets
\idb^{\cv{spec}_\cv{n}} \cup \set{\myname(\cv{n})}$
\Comment{$\cv{n}$'s initial DB gets the initial data fixed by its
  specification, plus its name}
\Else\ $\state'.\db(\cv{n}) \gets \state.\db(\cv{n})$
\Comment{Default case: persisting agent not affected by the interaction}
\EndIf
\EndFor
\If{ $\exists \state'' \in \Sigma$ s.t.~$\state''.\db(\inst) =
  \state'.\db(\inst)$ and for each $\tup{\cv{n},\_} \in \curag$,
  $\state''.\db(\cv{n}) = \state'.\db(\cv{n})$ }
\State $\trans \gets \trans \cup \tup{\state,\state''}$
\Comment{State already exists: connect $\state$ to that state}
\Else
\ $\Sigma \gets \Sigma \cup \set{\state'}$, 
$\trans \gets \trans \cup \tup{\state,\state'}$
\Comment{Add and connect new state}
\EndIf
\EndIf

\EndWhile
%\State \textbf{return} $\tup{\const,\tbox \cup \cbox,\Sigma,
%  s_0,\db,\trans}$
\EndProcedure
\end{algorithmic}
\vspace{-.3cm}
\caption{Simplification of \textsc{build-ts} dealing with
  shallow-typed RMASs}\label{alg:ts-sh}
\end{figure*}

% \caption{Procedure for constructing a transition system that is a
%   finite-branching, faithful abstraction of the one constructed by
%   \textsc{build-ts}}\label{alg:ts-fb}

%%% Local Variables:
%%% mode: latex
%%% TeX-master: "main"
%%% save-place: t
%%% End:

\subsection{RMASs with the Successor Relation}
We now show that including at least one data type with the successor
relation compromises decidability:

\begin{theorem}
Verification of a propositional reachability property over
state-bounded, shallow-typed RMASs that use a single data type equipped with the
successor relation is
undecidable, even when the RMAS contains a single agent that uses
unary relations only.
\end{theorem}

\begin{proof}
The proof is by reduction from the halting problem of two-counter
machines. 
A \emph{counter} is a memory
register that stores a (non-negative) integer. Notice that
the proof works in the same way even if we substitute $\mathbb{Z}$
with $\mathbb{Q}$ or $\mathbb{R}$, provided that they are equipped
with the successor relation.

Given two positive integers $n, m \in \mathbb{N}^+$, an
\emph{$m$-counter machine} $\cm$ with counters $c_1,\ldots,c_m$
is a program constituted by a (numbered) sequence of $n$ instructions:
\[
1: \mathit{CMD}_1; \quad 2: \mathit{CMD}_2; \quad \ldots \quad n: \halt;
\]
where the $n$-th instruction indicates that $\cm$ halts, while for
every $k \in \set{1,\ldots,n-1}$, instruction  $k:
\mathit{CMD}_k$ has one of the two following forms:
\begin{compactitem}
\item (\emph{increment command} for counter $i$) $\mathit{CMD}_k =
  \inc{i}{k'}$, with $i \in \set{1,\ldots,m}$ and $k' \in
  \set{1,\ldots,n}$, which increases the counter $c_i$ of one unit, and then
  jumps to instruction number $k'$:
\[
%k: \inc{i}{k'} \quad \text{ means } 
\quad k: c_i:=c_i+1;~\goto{k'};
\]
\item (\emph{conditional decrement command} for counter $i$)
$\mathit{CMD}_k =\cdec{i}{k'}{k''}$, with $i \in \set{1,\ldots,m}$ and $k',k'' \in
  \set{1,\ldots,n}$, which tests
  whether the value of counter $i$ is zero. If so, it jumps to
  instruction $k'$; otherwise, it decreases counter $i$ of one unit,
  and then jumps to instruction $k''$: 
\[k: \begin{array}[t]{rl}
\textbf{if}~c_i == 0 & \textbf{then}~\goto{k'}; \\
& \textbf{else}~\{ c_i:=c_i-1;~\goto{k''}; \}\\
\end{array}
\]
\end{compactitem}
An \emph{input} for an $m$-counter machine is an $m$-tuple of values
$\tup{d_1,\ldots,d_m}$ (such that $d_i \in \mathbb{N}$), used to
initialize its counters. Given an $m$-counter machine $\cm$ and an
input $I$ of size $m$, we say that $\cm$ \emph{halts on input}
$I$ if the execution of $\cm$ with counter initial values set by
$I$ eventually reaches the last, $\halt$ command.

It is well-known that checking whether a 2-counter machine halts on a
given input is undecidable \cite{Min67}, and that undecidability still
holds when checking whether the 2-counter machine halts on input
$\tup{0,0}$.

We show how to encode a 2-counter machine into a state-bounded, shallow-typed
RMAS containing a single agent specification that work over unary
relations only. Specifically, given a 2-counter machine $\cm$ with $n$
instructions, we
construct RMAS $\sys_\cm = \tup{\set{AT,ZT}, \set{AF,ZF}, \set{\cv{0},\ldots,\cv{k}},\set{\typed{\scname{input}}}, \set{\mathit{go}},
    \emptyset,\instspec_\cm}$, where $k = max\set{2,n}$, and:
\begin{compactitem}
\item $AT = \tup{\mathbb{A},\set{=}}$ is the agent type (just used to
  keep track of the $\inst$ name), $ZT =
  \tup{\mathbb{Z},\set{<,=,\succp}}$ is the integer type   (but, as specified above, $\mathbb{Z}$ can be seamlessly substituted
  by $\mathbb{Q}$ or $\mathbb{R}$).
\item $AF$ and $ZF$ are the base facets defined starting from $AT$ and
  $ZT$ respectively.
\item $\typed{\scname{input}} = \tup{\scname{input}/0,\tup{},ZF}$ is a
  0-ary service that returns integer values.
\item $\mathit{go}$ is a message sent by $\inst$ to itself so as to
  trigger the processing of the next instruction.
\item $\instspec_\cm$ is a specification for the institutional agent
  that mimics the program of $\cm$.
\end{compactitem}

Specifically,  $\instspec_\cm =
  \tup{\mathsf{instspec},\schema_\cm,\constr_\cm,\idb^\inst,\crule_\cm,\uact_\cm,\urule_\cm}$,
  where:
\begin{compactitem}
\item $\schema_\cm = \left\{
\begin{array}{@{}l@{}}
\typed{C_1}(ZF),\typed{C_1^{p}}(ZF),
  \typed{C_2}(ZF) ,\typed{C_2^{p}}(ZF),\\\typed{\ex{PC}}(ZF),\typed{\ex{Op}}(ZF),\typed{\ex{Target}}(ZF),\typed{\ex{Halted}}()\\
\typed{\agent}(AF),\typed{\myname}(AF)
\end{array}
\right\}
$ where:
\begin{compactitem}
\item $\typed{C_1}$ and $\typed{C_2}$ store the current values of the two counters, 
\item $\typed{C_1^{p}}$ and $\typed{C_2^{p}}$ store their previous values, 
\item $\typed{\ex{PC}}$ stores the program counter (i.e., the number of the instruction
to be processed),
\item $\typed{\ex{Op}}$ indicates the nature of the 
operator to be applied ($\cv{0}$ means increment, while $\cv{1}$ means decrement),
\item $\typed{\ex{Target}}$ indicates the target counter, that is, the
  counter to which the operation must be applied ($\cv{1}$ means the
  first counter, $\cv{2}$ the second),
\item $\typed{\ex{Halted}}$ is a proposition indicating that the
  agent finished the execution (i.e., reached the last instruction of $\cm$).
\end{compactitem}
\item $\constr_\cm$ contains constraints that encode the semantics of
  operations. In particular:
\begin{compactitem}
\item In the case of increment, the target counter must have a current
  value that is successor of the previous value:
\[
\begin{array}{@{}l@{}}
\ex{Op}(\cv{0}) \land \ex{Target}(\cv{1}) \\
\qquad \limp 
(\forall x_p,x. C_1(x) \land C_1^{p}(x_p) \limp \succp(x,x_p))\\
\ex{Op}(\cv{0}) \land \ex{Target}(\cv{2}) \\
\qquad \limp 
(\forall x_p,x. C_2(x) \land C_2^{p}(x_p) \limp \succp(x,x_p))
\end{array}
\]
\item In the case of decrement, the opposite holds, i.e., the target
  counter must have a current value that is precedessor of the
  previous value:
\[
\begin{array}{@{}l@{}}
\ex{Op}(\cv{1}) \land \ex{Target}(\cv{1}) \\
\qquad \limp 
(\forall x_p,x. C_1(x) \land C_1^{p}(x_p) \limp \succp(x_p,x))\\
\ex{Op}(\cv{1}) \land \ex{Target}(\cv{2}) \\
\qquad \limp 
(\forall x_p,x. C_2(x) \land C_2^{p}(x_p) \limp \succp(x_p,x))
\end{array}
\]
\end{compactitem}
\item The initial database of $\inst$ initializes the two counters to
  $\cv{0}$, and the program counter to the first instruction: 
\[
\small
\begin{array}{@{}l@{}}
\idb^\inst =
\set{\agent(\inst),\myname(\inst),C_1(\cv{0}),C_2(\cv{0}),PC(\cv{1})}
\end{array}
\]
\item $\crule_\cm$ contains just a single rule, which enables $\inst$
  to send a $\ex{go}$ message to itself if it is not halted:
\[
\myname(a) \land \neg \ex{Halted}~\ena~\ex{go}() ~\toa~ a
\]
\item $\uact_\cm$ contains the following actions:
\begin{compactitem}
\item $\typed{\actname{set-pc}}(ZF)$ updates the program
  counter to the value passed as parameter:
\[
\actname{set-pc}(\pname{next}):
\left\{
\begin{array}{@{}r@{~}l@{}l@{}}
\ex{PC}(x) \rightsquigarrow
&\del&\set{\ex{PC}(x)},\\
\true \rightsquigarrow  &\add&\set{\ex{PC}(\pname{next})}
\end{array}
\right\}
\]
\item $\typed{\actname{set-op}}(ZF,ZF)$ sets the operation,
  i.e., the operation type and the target counter, to the passed parameters:
\[
\actname{set-op}(\pname{o},\pname{t}):
\left\{
\begin{array}{@{}r@{~}l@{}l@{}}
\ex{Op}(x) \rightsquigarrow
&\del&\set{\ex{Op}(x)},\\
\ex{Target}(x) \rightsquigarrow
&\del&\set{\ex{Target}(x)},\\
\true \rightsquigarrow  &\add&\set{\ex{Op}(\pname{o})}\\
\true \rightsquigarrow  &\add&\set{\ex{Target}(\pname{t})}\\
\end{array}
\right\}
\]
\item $\typed{\actname{u-c}}(ZF)$ updates the value of the
  counter whose index is passed as parameter, and at the same time
  remembers the current value moving it to the ``previous'' counter relation:
\[
\hspace*{-.6cm}
\actname{u-c}(\pname{c}):
\left\{
\begin{array}{@{}l@{~}c@{~}l@{}l@{}}
\pname{c} = \cv{1} \land C_1^p(x)& \rightsquigarrow
&\del&\set{C_1^p(x)}\\
\pname{c} = \cv{1} \land C_1(x)  
&
\rightsquigarrow
&\del&\set{C_1(x)},
\add\set{C_1^p(x)}\\
\pname{c} = \cv{1} &\rightsquigarrow &\add&\set{C_1(\scname{input}())}\\
\pname{c} = \cv{2} \land C_2^p(x)& \rightsquigarrow
&\del&\set{C_2^p(x)}\\
\pname{c} = \cv{2} \land C_2(x)  
&
\rightsquigarrow
&\del&\set{C_2(x)},
\add\set{C_2^p(x)}\\
\pname{c} = \cv{2} &\rightsquigarrow &\add&\set{C_2(\scname{input}())}\\
\end{array}
\right\}
\]
It is worth noting that the action nondeterministically updates the
content of the first or second counter, depending on the value of the parameter. However, by considering the
constraints modelled in $\constr_\cm$, only the successor state that
has picked exactly the successor or precedessor value of the current
one will be selected, depending on what the current operation is.
\item $\typed{\actname{stop}}()$ is an action without parameters that just
  sets the $\ex{Halted}$ flag to true:
\[
\actname{stop}(): \set{\true \rightsquigarrow \add \set{\ex{Halted}}}
\]
\end{compactitem}
\item $\urule_\cm$ constains a set of rules that mirror the
  instructions of $\cm$, according from the following translation
  schema:
\begin{compactitem}
\item For instruction $k: \inc{i}{k'}$ (with $i \in \set{1,2}$), we
  get:
\[
\begin{array}{@{}l@{}}
\onev~\ex{go}~\ifc~\ex{PC}(\cv{k})~\thendo~
\actname{set-pc}(\cv{k}')\\
\onev~\ex{go}~\ifc~\ex{PC}(\cv{k})~\thendo~
\actname{set-op}(\cv{0},\cv{i})\\
\onev~\ex{go}~\ifc~\ex{PC}(\cv{k})~\thendo~
\actname{u-c}(\cv{i})\\
\end{array}
\]
The first rule handles the update of the program counter. The second
rule indicates that counter $\cv{i}$ must be subject to operation with
code $\cv{0}$. The third rule indicates that the instruction require
to update the content of counter $\cv{i}$.
\item For instruction $k: \cdec{i}{k'}{k''}$ (with $i \in \set{1,2}$), we
  get:
\[
\begin{array}{@{}l@{}}
\onev~\ex{go}~\ifc~\ex{PC}(\cv{k}) \land C_{\cv{i}}(\cv{0})~\thendo~
\actname{set-pc}(\cv{k}')\\
\onev~\ex{go}~\ifc~\ex{PC}(\cv{k}) \land \neg C_{\cv{i}}(\cv{0})~\thendo~
\actname{set-pc}(\cv{k}'')\\
\onev~\ex{go}~\ifc~\ex{PC}(\cv{k}) \land \neg C_{\cv{i}}(\cv{0})~\thendo~
\actname{set-op}(\cv{1},\cv{i})\\
\onev~\ex{go}~\ifc~\ex{PC}(\cv{k}) \land \neg C_{\cv{i}}(\cv{0})~\thendo~
\actname{u-c}(\cv{i})\\
\end{array}
\]
The formalization is specular to the case of increment, with the
proviso that the manipuation of the counter is triggered only if the
counter is not
$\cv{0}$.
\item For instruction $n: \halt$, we simply get:
\[
\onev~\ex{go}~\ifc~\ex{PC}(\cv{n})~\thendo~\actname{halt}()
\]
\end{compactitem}
\end{compactitem}
It is now apparent that $\cm$ halts on input $\tup{0,0}$ if and
only if $\ts_{\sys_\cm} \models \mu Z. (\ex{Halted}) \lor \DIAM{Z}$

\end{proof}

% It is then easy to see
% that a single agent incorporating these two update actions can
% straightforwardly implement a counter machine, where self-directed communicative
% rules determine the instruction to be executed next, and on-receive
% rules react by executing the increment/decrement, at the same time
% manipulating the program counter (stored in a unary relation).

\subsection{Densely-Ordered RMASs}
Given the previous undecidability result, we focus on dense orders. A
\emph{densely-ordered} RMAS only relies on data types equipped with
domain-specific equality $=$ and, possibly, total dense orders, as well as
corresponding facets. For this class of RMASs, we have:

%Our key, decidability result is:
\begin{theorem}
Verification of closed $\mulpd$ properties over state-bounded,
densely-ordered RMASs is decidable, and reducible to conventional,
finite-state model checking.
\end{theorem}

\newcommand{\curadom}[1]{\mathit{ADom}_s(#1)}

\newcommand{\assignres}[4]{\textsc{assign-res}_{#1}^{#4}(#2,#3)}

\begin{figure*}[!t]
\begin{algorithmic}[1]
\small
\Procedure{build-fb-ts-shallow}{$\widehat{\sys}$}%\Comment{The g.c.d. of a and b}
\\\textbf{input:} Shallow-typed RMAS $\widehat{X}  = \tup{\dtset, \widehat{\dtset},
  \cdom{\dfset},\widehat{\scset}, \widehat{\evset}}$, with $\dtset =
\set{\dt_u^1,\ldots,\dt_u^n} \cup \set{\dt_o^1,\ldots,\dt_o^m}$,
\textbf{output:} TS $\Upsilon_\X=\tup{\ddom{\dt},\Sigma,
  \state_0,\trans}$
\State $AS_{0} \gets \set{\tup{\cv{n},\cv{spec}_\cv{n}} \mid
  \hasSpec(\cv{n},\cv{spec}_\cv{n}) \in \idb^\inst}$ \Comment{Initial
  agents with their specifications}
\ForAll{$\tup{\cv{n},\cv{spec}_\cv{n}} \in AS_{0}$}
 $\state_0.\db(\cv{n}) \gets \idb^{\cv{spec}_\cv{n}}$
\Comment{Specify the initial state by extracting the initial DBs from
  the agent specs}
\EndFor
\State $\Sigma \gets \{\state_0\}$, $\trans \gets \emptyset$
\While{true}
\State \textbf{pick} $\state \in \Sigma$
\Comment{Nondeterministically pick a state}
\State $\curag \gets \set{\tup{\cv{n},\cv{spec}_\cv{n}} \mid
  \hasSpec(\cv{a},\cv{spec}_\cv{n}) \in \state.\db(\inst)}$
\Comment{Get currently active agents with their specifications}
\State \textbf{pick} $\tup{\cv{a},\cv{spec}_\cv{a}} \in \curag$\Comment{Nondeterministically
  pick an active agent $\cv{a}$, elected as ``sender''}
\State $\enaset \gets \Call{get-msgs}{\widehat{\cspec}^{\cv{spec}_\cv{a}},s.\db(\cv{a}),\curag}$ \Comment{Get the enabled messages with target agents}
\If{$\enaset \neq \emptyset$}
\State \textbf{pick} $\tup{M(\vec{\cv{o}}),\cv{b}} \in
\enaset$, with $\tup{\cv{b},\cv{spec}_\cv{b}} \in \curag$
\Comment{Pick a message+target agent and trigger message
  exchange and reactions}
\State $\aset_\cv{a} \gets \emptyset$, $\aset_\cv{b} \gets \emptyset$
\Comment{Get the actions with actual parameters to be applied by
  $\cv{a}$ and $\cv{b}$}
\ForAll{matching on-send rules $``\onev~ M(\vec{x})~\toa~ t~\ifc~ Q(t,\vec{x}) ~\thendo~
  \alpha(t,\vec{x})"$ in $\widehat{\urule}^{\cv{spec}_\cv{a}}$} 
\If{$\ans{Q(\cv{b},\vec{\cv{o}}},\state.\db(\cv{a}))$} 
$\aset_\cv{a} \gets \aset_\cv{a} \cup \alpha(\cv{b},\vec{\cv{o}})$
\label{alg:acta-sh}
\EndIf
\EndFor
\ForAll{matching on-receive rules $``\onev~ M(\vec{x})~\froma~ s~\ifc~ Q(s,\vec{x}) ~\thendo~
  \alpha(s,\vec{x})"$  in $\widehat{\urule}^{\cv{spec}_\cv{b}}$} 
\If{$\ans{Q(\cv{a},\vec{\cv{o}}},\state.\db(\cv{b}))$} 
$\aset_\cv{b} \gets \aset_\cv{b} \cup \alpha(\cv{a},\vec{\cv{o}})$
\label{alg:actb-sh}
\EndIf
\EndFor
\State $\tup{\torem^\cv{a},\toaddsk^\cv{a}} \gets
\Call{get-facts}{\widehat{\sys},s.\db(\cv{a}),\aset_{\cv{a}}}$, 
 $\tup{\torem^\cv{b},\toaddsk^\cv{b}} \gets
\Call{get-facts}{\widehat{\sys},s.\db(\cv{b}),\aset_{\cv{b}}}$
\State $\dbvar_s^{\cv{a}} \gets (s.\db(\cv{a}) \setminus
\torem^{\cv{a}}) \cup \toaddsk^{\cv{a}}$ \Comment{Calculate new
  $\cv{a}$'s DB, still with service calls to be issued}
\State $\dbvar_s^{\cv{b}} \gets (s.\db(\cv{b}) \setminus
\torem^{\cv{b}}) \cup \toaddsk^{\cv{b}}$ \Comment{Calculate new
  $\cv{b}$'s DB, still with service calls to be issued}
\ForAll{data type $\dt \in \dtset$} 
\Comment{Fetch the active domain and service calls for each type}
\State $\curadom{\dt} \gets 
\begin{array}{@{}l@{}l@{}}
&\set{\cv{d}
  \mid \cv{d} \in \tdom{\dt} \cap \cdom{\dfset}}\\
\cup 
% &\left\{\cv{d}
%   \mid \cv{d} \in \tdom{\dt} \cap \left(\dbvar_s^{\cv{a}} \cup
%     \dbvar_s^{\cv{b}} \cup \bigcup_{\tup{\cv{n},\cv{spec}_\cv{n}} \in
%     \curag\setminus\set{\tup{\cv{a},\cv{spec}_\cv{a}},\tup{\cv{b},\cv{spec}_\cv{b}}}}\adom{\state.\db(\cv{n})}\right)
% \right\}
&\left\{\cv{d}
  \mid \cv{d} \in \tdom{\dt} \cap \adom{\state}
\right\}\\
\cup &
\set{\scname{f}(\vec{\cv{o}}) \mid \scname{f}(\vec{\cv{o}}) \in \calls{\dbvar_s^{\cv{a}} \cup
    \dbvar_s^{\cv{b}} } \text{ and } \typed{\scname{f}} =
  \tup{\scname{f}/n,\dfset^{in},\df^{out}} \in \widehat{\scset} \text{
  with } \df^{out} = \tup{\dt,\true}} 
\end{array}$
\EndFor
\State \textbf{pick} 
$\mathfrak{H} \in 
\left\{\tup{\P_1,\ldots,\P_n,\H_1,\ldots,\H_m} \left|
\begin{array}{@{}l@{}}
\P_i \text{ is a }
T_u^i\text{-equality commitment on } \curadom{T_u^i} \text{ for } i
\in \set{1,\ldots,n},\\
\H_j \text{ is a }
T_o^j\text{-densely ordered commitment on } \curadom{T_o^j} \text{ for } j
\in \set{1,\ldots,m}
\end{array}
\right.
\right\}
$
\State $ \sigma \gets \left\{\scname{f}(\vec{\cv{o}}) \mapsto \cv{d}
  \mid \scname{f}(\vec{\cv{o}}) \in \callsvar \text{ and }
  \assignres{\mathfrak{H}}{\state}{\scname{f}(\vec{\cv{o}})}{\ddom{\dtset}} = \cv{d}\right\}$
\label{alg:sc-sh}
\State $\dbvar^{\cv{a}}_{\mathit{cand}} \gets \dbvar_s^{\cv{a}}\sigma, \dbvar^{\cv{b}}_{\mathit{cand}}
\gets \dbvar_s^{\cv{b}}\sigma$
\Comment{Obtain new candidate DBs by substituting service calls with
  results}
\If{$\dbvar^{\cv{a}}_{\mathit{cand}}$ satisfies
  $\widehat{\constr}^{\cv{a}}$}
$\dbvar_{\cv{a}} \gets \dbvar^{\cv{a}}_{\mathit{cand}}$
\Comment{Update $\cv{a}$'s DB}
\label{alg:commita-sh}
\Else\ 
$\dbvar_{\cv{a}} \gets \state.\db(\cv{a})$
\Comment{Rollback $\cv{a}$'s DB}
\EndIf
\If{$\dbvar^{\cv{b}}_{\mathit{cand}}$ satisfies
  $\widehat{\constr}^{\cv{b}}$}
$\dbvar_{\cv{b}} \gets \dbvar^{\cv{b}}_{\mathit{cand}}$
\Comment{Update $\cv{b}$'s DB}
\label{alg:commitb-sh}
\Else\ 
$\dbvar_{\cv{b}} \gets \state.\db(\cv{b})$
\Comment{Rollback $\cv{b}$'s DB}
\EndIf
\State \textbf{pick} fresh state $\state'$ \Comment{Create new state}
\State $\newag \gets \emptyset$ \Comment{Determine the (possibly
  changed)   set of active agents and their specs}
\If{$\cv{a} = \cv{\inst}$}
 $\newag \gets \set{\tup{\cv{n},\cv{spec}_\cv{n}} \mid \hasSpec(\cv{n},\cv{spec}_\cv{n}) \in \dbvar_{\cv{a}}}$
\ElsIf{$\cv{b} = \cv{\inst}$}
 $\newag \gets \set{\tup{\cv{n},\cv{spec}_\cv{n}} \mid \hasSpec(\cv{n},\cv{spec}_\cv{n}) \in \dbvar_{\cv{b}}}$
\Else\ $\newag \gets \curag$
\Comment{No change if $\inst$ is not involved in the interaction or
  must reject the update}
\EndIf
\ForAll{$\tup{\cv{n},\cv{spec}_\cv{n}} \in \newag$} \Comment{Do the
  update for each active agent}
\If{$\cv{n} = \cv{a}$}
  $\state'.\db(\cv{n}) \gets \dbvar_{\cv{a}}$ 
  \Comment{Case of sender agent} 
\ElsIf{$\cv{n} = \cv{b}$} 
  $\state'.\db(\cv{n}) \gets \dbvar_{\cv{b}}$
\Comment{Case of target agent}
\ElsIf{$\cv{n} \not\in \curag$} 
\Comment{Case of newly created agent}
\State $\state'.\db(\cv{n}) \gets
\idb^{\cv{spec}_\cv{n}} \cup \set{\myname(\cv{n})}$
\Comment{$\cv{n}$'s initial DB gets the initial data fixed by its
  specification, plus its name}
\Else\ $\state'.\db(\cv{n}) \gets \state.\db(\cv{n})$
\Comment{Default case: persisting agent not affected by the interaction}
\EndIf
\EndFor
\If{ $\exists \state'' \in \Sigma$ s.t.~$\state''.\db(\inst) =
  \state'.\db(\inst)$ and for each $\tup{\cv{n},\_} \in \curag$,
  $\state''.\db(\cv{n}) = \state'.\db(\cv{n})$ }
\State $\trans \gets \trans \cup \tup{\state,\state''}$
\Comment{State already exists: connect $\state$ to that state}
\Else
\ $\Sigma \gets \Sigma \cup \set{\state'}$, 
$\trans \gets \trans \cup \tup{\state,\state'}$
\Comment{Add and connect new state}
\EndIf
\EndIf
\EndWhile
\EndProcedure

\end{algorithmic}
\vspace{-.3cm}

\caption{Procedure for constructing a transition system that is a
   finite-branching, faithful abstraction of the transition system constructed by
   \textsc{build-ts-shallow} }\label{alg:ts-fb}

\end{figure*}

%%% Local Variables:
%%% mode: latex
%%% TeX-master: "main"
%%% save-place: t
%%% End:

Let $\sys = \rmas$ be an RMAS, and $\Phi$ be a closed $\mulpd$
property.
Notice that, by hypothesis, $\dtset$ is
constituted by a set $\dtset_u$ of data types equipped with
domain-specific equality only, and a set $\dtset_o$ of data
types equipped also with a dense total order: $\dtset = \dtset_u \uplus \dtset_o$.

The proof is quite involved, so we separate it into several steps and
intermediate lemmas.

%  The roadmap of the proof is as follows:
% \begin{compactenum}
% \item \emph{(Shallowing)} We reformulate
% the input RMAS $\sys$ into the equivalent, shallow-typed version
% $\widehat{\sys} = \tup{\dtset, \widehat{\dtset},
%   \cdom{\dfset},\widehat{\scset}, \widehat{\evset}}$, as defined in the
%   proof of Theorem~\ref{thm:facet}. By Theorem~\ref{thm:facet}, we
%   have that $\ts_\sys \models \Phi$ if and only if $\ts_{\widehat{\sys}}
%   \models \Phi$.
% \item \emph{(Pruning)}

%  % Following the approach in \cite{MoCD14}, we
%  %  transform $\widehat{\sys}$ into another shallow-typed RMAS that
%  %  unifies all agent databes into
% \end{compactenum}

% Then, given that for non-ordered domains the proof resembles that of
% \cite{BCDDM13,BCDDM13b,MoCD14}, we focus, from now, only data types
% equipped both with a domain-specific equality relation $=$ and total, dense
% ordering relation $<$. Notice also that our proof directly carries
% over non-ordered domains as well.

The first step consists in reformulating 
the input RMAS $\sys$ into the equivalent, shallow-typed version
$\widehat{\sys} = \tup{\dtset, \widehat{\dtset},
  \cdom{\dfset},\widehat{\scset}, \widehat{\evset}}$, as defined in the
  proof of Theorem~\ref{thm:facet}. By Theorem~\ref{thm:facet}, we
  have that $\ts_\sys \models \Phi$ if and only if $\ts_{\widehat{\sys}}
  \models \Phi$.

As a second step, we consider the infinite-state
transition system $\ts_{\widehat{\sys}}$, and seek a faithful (sound
and complete) finite-state abstraction of it, suitably extending the
technique in \cite{BCDDM13} so as to consider types, and dense orders
in particular. Since $\sys$ is
state-bounded, two sources of infinity are possibly present in $\ts_{\sys}$
and $\ts_{\widehat{\sys}}$:
\begin{compactenum}
\item infinite branching, that is, presence of a state with infinitely many successors due to
  the injection of data through service calls;
\item infinite runs, that is, runs that visit infinitely many
  different agent databases. 
\end{compactenum}

We can get rid of the infinite-branching in $\ts_{\widehat{\sys}}$ by
suitably pruning it:
\begin{lemma} 
\label{lem:fb}
For every shallow-typed RMAS $\widehat{\sys}$, there exists a transition system $\pts_{\widehat{\sys}}$ that obeys
the following properties:
\begin{compactenum}[\it (i)]
\item $\pts_{\widehat{\sys}}$ is finite-branching;
\item for every closed $\mulpd$ property $\Phi$, $\ts_{\widehat{\sys}}\models
  \Phi$ if  and only if $\pts_{\widehat{\sys}} \models \Phi$.
\end{compactenum}
\end{lemma}

To produce $\pts_{\widehat{\sys}}$, we extend the notion of \emph{equality
 commitment} exploited in \cite{BCDDM12,BCDDM13}. 
Equality commitments are used to abstractly describe how
the result of a service call relates through (in)equality to the values present in the
active domain of the system, and to those returned by other service
calls issued in the same moment, without considering their actual, specific
results. 
 Technically, we adapt the definition of equality commitment in
 \cite{BCDDM12} to the case of RMASs, taking into account that:
\begin{inparaenum}[\it (i)]
\item differently from DCDSs, data objects are typed, and
\item some data objects could be compared not only with equality, but
  also with a domain-specific total, dense relation.
 %(which will be  depicted as $<$ in the following).
\end{inparaenum}

Consider a data type $\dt_u \in \dtset_u$, and a set $S$ made up of data
objects in $\tdom{\dt_u}$ and of ground service calls built by applying a
service call $\typed{\scname{f}} \in \scset$ to input data objects,
such that the return type of $\typed{\scname{f}}$ is compatible with
$\dt_u$. A \emph{$\dt_u$-equality commitment} $\P$ on $S$ is a partition of $S$, that
is, a set of disjoint subsets of $S$, called \emph{cells}, such that
the union of the cells in $\P$ is exactly $S$. Each cell contains at
most one data object (but arbitrarily many ground service calls). For
any $e \in \P$, $\cell{e}{\P}$ denotes the cell to which $e$ belongs.

The intention of $\H$ is to abstractly
characterize how the elements in $S$ are related to each other via the
domain-specific relation $=_{\dt_u}$ of $\dt_u$. In particular, $\P$ is
used to capture equality and
non-equality commitments on the members of $S$ in the following sense:
for every $e_1,e_2 \in S$, we have $e_1 =_{\dt_u} e_2$ if and only if
$\cell{e_1}{\H} =_{\dt_u} \cell{e_2}{\H}$.

Consider now a data type $\dt_o \in \dtset_o$, and a set $S$ as before. 
A \emph{$\dt_o$-densely ordered commitment} $\H$ on $S$ is a pair
$\tup{\P,\ex{pos}}$, where:
\begin{compactitem}
\item $\P$ is a $\dt_o$-equality commitment over $S$;
\item $\ex{pos}$ is an ordering over $\P$ that is compatible with $S$,
  i.e., $\ex{pos}$ is a bijection $\set{1,\ldots,|\P|} \longrightarrow
  \P$ that obeys to the following property: for every $P_1,P_2 \in
  \P$, whenever $P_1$ contains a data object $\cv{d}_1 \in \dt$ and $P_2$ contains a data object $\cv{d}_2$ in $\tdom{\dt}$ , we have $\ex{pos}(P_1)
  <_{\mathbb{N}} \ex{pos}(P_2)$ if and only if $\cv{d}_1 <_{\dt_o} \cv{d}_2$,
  where $<_{\mathbb{N}}$ denotes the total order relation on natural
  numbers. 
\end{compactitem}

The intention of $\H$ is to abstractly
characterize how the elements in $S$ are related to each other via the
domain-specific relations $=_{\dt_o}$ and $<_{\dt_o}$ of $\dt$. Specifically, $\P$
covers equality, while $\ex{pos}$ accounts for $<$, and orders the
  members of $S$ in the following sense: for every $e_1,e_2 \in S$, we
  have $e_1 <_{\dt_o} e_2 $ if and only if $\ex{pos}(\cell{e_1}{\P}) <_\mathbb{N}
  \ex{pos}(\cell{e_2}{\P})$.

We now exploit commitments to change the
$\textsc{build-ts}$ algorithm, shown in Figure~\ref{alg:ts} and used
to construct $\ts_{\widehat{\sys}}$. In particular, we start from the
\textsc{ts-build-shallow} procedure, and modify the function that
nondeterministically selects the results returned by service
calls. First of all, we assume the existence of a pre-defined
function $\textsc{assign-res}$, parameterized by a tuple of commitments, which substitutes a service call
with a corresponding result that is in accordance with the cell to
which the service call belongs. In particular, let $\dtset_u =
\set{\dt_u^1,\ldots,\dt_u^n}$ and $\dtset_o =
\set{\dt_o^1,\ldots,\dt_o^m}$. Let
$\tup{S_u^1,\ldots,S_u^n,S_o^1,\ldots,S_o^m}$ be a tuple of sets, each
containing data objects from the corresponding type, and possibly also
service calls whose return type matches with that type.  Let
$\mathfrak{H} = \tup{\P_1,\ldots,\P_n,\H_1,\ldots,\H_m}$ be a tuple of
commitments, where each $\P_i$ is a $T_u^i$-equality commitment built over
$S_u^i$, and where each $\H_j$ is a $T_o^j$-densely ordered commitment
built over $S_o^j$. 

Specifically, given a data domain $\Delta$, we define
\[\textsc{assign-res}_{\mathfrak{H}}^\Delta: \Sigma \times
\calls{\bigcup_{i \in \set{1,\ldots,n}} \hspace*{-.5cm} S_u^i \cup \bigcup_{j \in
    \set{1,\ldots,m}} \hspace*{-.5cm} S_o^j } \longrightarrow \Delta\]
where, by fixing a state $\state \in \Sigma$,
$\textsc{assign-res}_{\mathfrak{H}}^\Delta$ obeys to the following properties:
\begin{compactitem}
\small
\item For $i \in \set{1,\ldots,n}$, for every service call $\scname{f}(\vec{\cv{o}}) \in S_u^i$
  and for every data object $\cv{d} \in S_u^i$:
  $\assignres{\mathfrak{H}}{\state}{\scname{f}(\vec{\cv{o}})}{\Delta} =_{T_u^i}
  \cv{d}$ iff $\cell{\scname{f}(\vec{\cv{o}})}{\P_i} =_{T_u^i}
  \cell{\cv{d}}{\P_i}$.
\item For $i \in \set{1,\ldots,n}$ and for every two service
  calls $\scname{f_1}(\vec{\cv{o_1}}),\scname{f_2}(\vec{\cv{o_2}}) \in
  S_u^i$: 
$ \assignres{\mathfrak{H}}{\state}{\scname{f_1}(\vec{\cv{o_1}})}{\Delta} =_{T_u^i}
 \assignres{\mathfrak{H}}{\state}{\scname{f_2}(\vec{\cv{o_2}})}{\Delta}$
   iff  $\cell{\scname{f_1}(\vec{\cv{o_1}})}{\P_i} =_{T_u^i}
 \cell{\scname{f_2}(\vec{\cv{o_2}})}{\P_i}$.
\item For $j \in \set{1,\ldots,m}$ with $\H_j = \tup{\P'_j,\ex{pos}_j}$, for every service call $\scname{f}(\vec{\cv{o}}) \in S_o^j$
  and for every data object $\cv{d} \in S_o^j$:
  $ \assignres{\mathfrak{H}}{\state}{\scname{f}(\vec{\cv{o}})}{\Delta} =_{T_u^i}
  \cv{d}$ iff $\cell{\scname{f}(\vec{\cv{o}})}{\P'_j}{\Delta} =_{T_u^i}
  \cell{\cv{d}}{\P'_j}$.
\item For $\H_j =
  \tup{\P'_j,\ex{pos}_j}$ ( $j \in \set{1,\ldots,m}$), and for every two service calls
  $\scname{f_1}(\vec{\cv{o_1}}),\scname{f_2}(\vec{\cv{o_2}}) \in
  S_o^j$: 
        $\assignres{\mathfrak{H}}{\state}{\scname{f_1}(\vec{\cv{o_1}})}{\Delta} =_{T_o^j}
 \assignres{\mathfrak{H}}{\state}{\scname{f_2}(\vec{\cv{o_2}})}{\Delta}$
   iff $\cell{\scname{f_1}(\vec{\cv{o_1}})}{\P'_j} =_{T_o^j}
 \cell{\scname{f_2}(\vec{\cv{o_2}})}{\P'_j}$.
\item For $\H_j =
  \tup{\P'_j,\ex{pos}_j}$ ( $j \in \set{1,\ldots,m}$), and for every two service calls
  $\scname{f_1}(\vec{\cv{o_1}}),\scname{f_2}(\vec{\cv{o_2}}) \in
  S_o^j$: 
 \begin{compactitem}
 \item 
$\assignres{\mathfrak{H}}{\state}{\scname{f_1}(\vec{\cv{o_1}})}{\Delta} =_{T_o^j}
 \assignres{\mathfrak{H}}{\state}{\scname{f_2}(\vec{\cv{o_2}})}{\Delta}$
   iff  $\cell{\scname{f_1}(\vec{\cv{o_1}})}{\P'_j} =_{T_o^j}
 \cell{\scname{f_2}(\vec{\cv{o_2}})}{\P'_j}$;
\item $\assignres{\mathfrak{H}}{\state}{\scname{f_1}(\vec{\cv{o_1}})}{\Delta} <_{T_o^j}
 \assignres{\mathfrak{H}}{\state}{\scname{f_2}(\vec{\cv{o_2}})}{\Delta}$ 
   iff $\ex{pos}(\cell{\scname{f_1}(\vec{\cv{o_1}})}{\P'_j}) <_{\mathbb{N}}
 \ex{pos}(\cell{\scname{f_2}(\vec{\cv{o_2}})}{\P'_j})$.
\end{compactitem}
\end{compactitem}
Intuitively, this function is used to select a \emph{single},
representative combination of service call results that obey to the
constraints imposed by a given commitment.

Figure~\ref{alg:ts-fb} shows the revised version of the algorithm in
Figure~\ref{alg:ts-sh}. Instead of picking any combination of service
call results, the $\textsc{build-fb-ts-shallow}$ algorithm picks an
equality/densely-ordered commitment for each type of the input RMAS,
constructed over the current active domain for that type, where the
current active domain for type $\dt$ is obtained by considering:
\begin{compactitem}
\item the initial data objects for $\dt$;
\item the current data objects for $\dt$;
\item the service calls that must be issued now, and whose return
  facet is defined over type $\dt$. 
\end{compactitem}
The combination of service call results for each type is then obtained
by applying the pre-defined $\textsc{assign-res}$ function.

Let  $\pts_{\widehat{\sys}}$ be the transition system obtained by the
application of the $\textsc{build-fb-ts-shallow}$ procedure over the
shallow-typed RMAS $\widehat{\sys}$. 
We first argue that $\pts_{\widehat{\sys}}$ is finite-branching,
differently from $\ts_{\widehat{\sys}}$, for which the function
$\textsc{get-call-res}$ may return infinitely many combinations of
service call results. In fact, given the current active domain
$\curadom{\dt}$ of a
type $\dt$, there are only finitely many commitments that can be
constructed over that set. More specifically, when $\dt$ is an
unordered type their number is bounded
by the Bell number of $|\curadom{\dt}|$, wherease when $\dt$ is an
ordered type their number is bounded
by the Bell number of $|\curadom{\dt}|$, multiplied by the factorial
of $|\curadom{\dt}|$ (so as to account for the permutation of data
objects). 
Since the $\textsc{assign-res}$ function assigns a single combination
of results for each commitment, there are only finitely many
combination of service call results, and consequently only finitely
many successor states of a given state can be present in
$\pts_{\widehat{\sys}}$.

To show that $\ts_{\widehat{\sys}}$ and $\pts_{\widehat{\sys}}$
satisfy the same set of $\mulpd$ formulae, one needs to follow
step-by-step the proof of \cite{BCDDM12,BCDDM13}, noticing that the
notion of densely-ordered commitment covers the case of formulae of
the form $x < y$, which is the only one not already tackled by \cite{BCDDM12,BCDDM13}.
This concludes the proof of Lemma~\ref{lem:fb}.

We now observe that $\pts_{\widehat{\sys}}$ may still contain runs visiting infinitely
many different states. The third phase of our proof consequently
consists of showing that it is possible to produce a ``folded''
folded transition system $\fts_{\widehat{\sys}}$ that
% obeys to the following properties:
is finite-state, and such that for every closed $\mulpd$ property $\Phi$,
$\pts_{\widehat{\sys}}\models \Phi$ if and only if $\fts_{\widehat{\sys}}
\models \Phi$.

Before showing how this can be done, we define a variant of
$\textsc{build-fb-ts-shallow}$ that, instead of employing the
domain-specific (rigid) ordering relations, relies on additional
``comparison tables'' that are suitably manipulated state by
state. The algorithm is shown in Figure~\ref{alg:ts-order}. The
construction algorithm exploits a specific database (indexed in the
state by symbol $<$) to store the projection of the ordering relations
of types in $\dtset_o$, where only actively persisting data objects
are considered. Such database employs a relation $\lessthan{\dt_o}$
for each densely-ordered data type $\dt_o \in \dtset_o$. In order to make the
input RMAS insisting on such relations instead of the domain-specific
ones, we introduce the $\flatten$ operator, which takes an RMAS or one
of its components, and substitutes every occurrence of a query of the
form $x <_{\dtset_o} y$ with the corresponding atomic query $\lessthan{\dt_o}(x,y)$.

Such a database is
initialized by computing, for
each data type $\dt_o^i \in \dtset_o$, the transitive closure of the
$<_{\dt_o^i}$ relation on the initial data domain for $\dt_o^i$, and by inserting all extracted
pairs into the dedicated $\lessthan{\dt_o^i}$ binary relation. It is
then used whenever a query is issued over an agent database, so as to
complement it with the explicit listing of all $\ex{lessThan}$
relations. Finally, it is updated state-by-state:
\begin{compactitem}
\item on the one hand by considering the issued service calls,
in accordance with the $\ex{pos}$ relation of the established
densely-ordered commitments (cf.~line~\ref{alg:lt-update} in Figure~\ref{alg:ts-order});
\item on the other hand by filtering away those tuples that involve a
  data object that is not persisting when performing a transition from
  the current to the next state (cf.~line~\ref{alg:lt-filter} in Figure~\ref{alg:ts-order}).
\end{compactitem}

Let $\pts^{flat}_{\widehat{\sys}}$ be the transition system produced
by $\textsc{build-fb-ts-shallow-flat}(\widehat{\sys})$. We have that:
\begin{lemma}
\label{lem:order}
For every shallow-typed RMAS $\widehat{\sys}$ and every closed
$\mulpd$ property $\Phi$: 
\[\pts_{\widehat{\sys}} \models \Phi \text{ if and
only if } \pts^{flat}_{\widehat{\sys}} \models \flatten(\Phi)\]
\end{lemma}
The lemma can be proven by induction on the construction of the two
transition systems, recalling that:
\begin{compactitem}
\item Every execution step of an RMAS is triggered by issuing
  domain-independent queries over the current database of one of its
  agents, and therefore comparisons can only be applied to data
  objects actively present in that databse.
\item $\mulpd$ queries can only compare data objects that are present
  in the current active domain of the system, or that were present in
  the immediately previous state. This is suitably handled, for
  $\flatten(\Phi)$, in line \ref{alg:lt-filter} of
  Figure~\ref{alg:ts-order}, where all comparisons between data
  objects present in the previous or current states are explicitly
  maintained. 
\end{compactitem}
It is also important to observe that $\pts^{flat}_{\widehat{\sys}}$
does not alter the state-boundedness of $\pts_{\widehat{\sys}}$,
because it only adds relations on data objects that are present in the
current or previous active domains, while comparisons between old data
objects are filtered away.

However, the crucial property of the construction of
$\pts^{flat}_{\widehat{\sys}}$, is that apart from data objects
present in the initial data domain, \emph{the comparison database is not based on
  the domain-specific ordering relations, but is constructed starting
  from the picked densely-ordered commitments}, as shown in line~\ref{alg:lt-update} of Figure~\ref{alg:ts-order}.
We combine this crucial property with the inability of $\mulpd$, due
to its persistent nature, of comparing currently active data objects with objects that were encountered in
 the past, but are not active anymore.  In particular, we can directly
 apply the data recycling technique in \cite{BCDDM12,BCDDM13}, reusing
 old, forgotten data objects in place of fresh ones. 

Figure~\ref{alg:ts-rcycl} shows the construction algorithm with
recycling of data objects. Let $\fts_{\widehat{\sys}}$ be the
transition system produced by such an algorithm. Due to the fact,
argued before, that during the system construction comparisons are
stored by analyzing densely-ordered commitments, and not 
domain-specific ordering relations, correctness is obtained
by adapting the proof in \cite{BCDDM12,BCDDM13}. In particular, we
obtain that, when the original RMAS is state-bounded, then only a
bounded number of new data objects
must be inserted before recycling makes it not necessary anymore to consider
fresh values, that is, before the set $\mathit{Passive}$ is guaranteed
to contain sufficiently many used but non-active data objects.  This
implies that the construction algorithm of Figure~\ref{alg:ts-rcycl}
terminates, and in turn that  $\fts_{\widehat{\sys}}$ is finite-state,
and represents at the same time a sound and complete abstraction of
the original system.

By putting everything together, we obtain in
fact that, for every state-bounded, densely-ordered RMAS $\sys$, and
for every $\mulpd$ property $\Phi$:
\begin{compactenum}
\item $\fts_{\widehat{\sys}}$ can be effectively constructed using the
  procedure $\textsc{build-ts-abstract}$ of Figure~\ref{alg:ts-rcycl};
\item $ \fts_{\widehat{\sys}}$ has a finite number of states;
\item $\ts_\sys \models \Phi \text{ if and only if } \fts_{\widehat{\sys}}
\models \flatten(\Phi)$.
\end{compactenum}
This concludes the proof.
\qed

\newcommand{\idblt}{\ex{D}_0^<}
\newcommand{\dblt}{\ex{D}^<}
\newcommand{\dbltplus}{\ex{D}_+^<}

\begin{figure*}[!t]
\begin{algorithmic}[1]
\small
\Procedure{build-fb-ts-shallow-flat}{$\widehat{\sys}$}%\Comment{The g.c.d. of a and b}
\\\textbf{input:} Shallow-typed, RMAS $\widehat{X}  = \tup{\dtset, \widehat{\dtset},
  \cdom{\dfset},\widehat{\scset}, \widehat{\evset}}$, with $\dtset =
\set{\dt_u^1,\ldots,\dt_u^n} \cup \set{\dt_o^1,\ldots,\dt_o^m}$\\
\textbf{output:} transition system $\Upsilon_\X=\tup{\ddom{\dt},\Sigma,
  \state_0,\trans}$
\State $\idblt \gets \emptyset$ \Comment{Initial DB incorporating the
  domain-specific $<$ relations for data objects in $\cdom{\dfset}$}
\ForAll{$i \in \set{1,\ldots,m}$}
\ForAll{$\cv{d_1},\cv{d_2} \in \cdom{\dfset} \cap \ddom{\dt_o^m}$}
\If{$\cv{d_1} <_{\dt_o^m}\cv{d_2}$} 
$\idblt \gets \idblt \cup \set{\lessthan{\dt_o^m}(\cv{d_1},\cv{d_2})}$
\EndIf
\EndFor
\EndFor
\State $AS_{0} \gets \set{\tup{\cv{n},\cv{spec}_\cv{n}} \mid
  \hasSpec(\cv{n},\cv{spec}_\cv{n}) \in \idb^\inst}$ \Comment{Initial
  agents with their specifications}
\ForAll{$\tup{\cv{n},\cv{spec}_\cv{n}} \in AS_{0}$}
 $\state_0.\db(\cv{n}) \gets \idb^{\cv{spec}_\cv{n}}$
\Comment{Specify the initial state by extracting the initial DBs from
  the agent specs}
\EndFor
\State $\state_0.\db(<) \gets \idblt$ \Comment{Insert the special
  less-than DB} 
\State $\Sigma \gets \{\state_0\}$, $\trans \gets \emptyset$
\While{true}
\State \textbf{pick} $\state \in \Sigma$
\Comment{Nondeterministically pick a state}
\State $\curag \gets \set{\tup{\cv{n},\cv{spec}_\cv{n}} \mid
  \hasSpec(\cv{a},\cv{spec}_\cv{n}) \in \state.\db(\inst)}$
\Comment{Get currently active agents with their specifications}
\State \textbf{pick} $\tup{\cv{a},\cv{spec}_\cv{a}} \in \curag$\Comment{Nondeterministically
  pick an active agent $\cv{a}$, elected as ``sender''}
\State $\enaset \gets
\Call{get-msgs}{\flatten(\widehat{\cspec}^{\cv{spec}_\cv{a}}),s.\db(\cv{a}) \cup
  \state.\db(<),\curag}$ \Comment{Get the enabled messages with target agents}
\If{$\enaset \neq \emptyset$}
\State \textbf{pick} $\tup{M(\vec{\cv{o}}),\cv{b}} \in
\enaset$, with $\tup{\cv{b},\cv{spec}_\cv{b}} \in \curag$
\Comment{Pick a message+target agent and trigger message
  exchange and reactions}
\State $\aset_\cv{a} \gets \emptyset$, $\aset_\cv{b} \gets \emptyset$
\Comment{Get the actions with actual parameters to be applied by
  $\cv{a}$ and $\cv{b}$}
\ForAll{matching on-send rules $``\onev~ M(\vec{x})~\toa~ t~\ifc~ Q(t,\vec{x}) ~\thendo~
  \alpha(t,\vec{x})"$ in $\flatten(\widehat{\urule}^{\cv{spec}_\cv{a}})$} 
\If{$\ans{Q(\cv{b},\vec{\cv{o}}},\state.\db(\cv{a})\cup \state.\db(<))$} 
$\aset_\cv{a} \gets \aset_\cv{a} \cup \alpha(\cv{b},\vec{\cv{o}})$
\EndIf
\EndFor
\ForAll{matching on-receive rules $``\onev~ M(\vec{x})~\froma~ s~\ifc~ Q(s,\vec{x}) ~\thendo~
  \alpha(s,\vec{x})"$  in $\flatten(\widehat{\urule}^{\cv{spec}_\cv{b}})$} 
\If{$\ans{Q(\cv{a},\vec{\cv{o}}},\state.\db(\cv{b}) \cup \state.\db(<))$} 
$\aset_\cv{b} \gets \aset_\cv{b} \cup \alpha(\cv{a},\vec{\cv{o}})$
\EndIf
\EndFor
\State $\tup{\torem^\cv{a},\toaddsk^\cv{a}} \gets
\Call{get-facts}{\flatten(\widehat{\sys}),s.\db(\cv{a})\cup
  \state.\db(<),\aset_{\cv{a}}}$
\State
 $\tup{\torem^\cv{b},\toaddsk^\cv{b}} \gets
\Call{get-facts}{\flatten(\widehat{\sys}),s.\db(\cv{b})\cup \state.\db(<),\aset_{\cv{b}}}$
\State $\dbvar_s^{\cv{a}} \gets (s.\db(\cv{a}) \setminus
\torem^{\cv{a}}) \cup \toaddsk^{\cv{a}}$ \Comment{Calculate new
  $\cv{a}$'s DB, still with service calls to be issued}
\State $\dbvar_s^{\cv{b}} \gets (s.\db(\cv{b}) \setminus
\torem^{\cv{b}}) \cup \toaddsk^{\cv{b}}$ \Comment{Calculate new
  $\cv{b}$'s DB, still with service calls to be issued}
\ForAll{data type $\dt \in \dtset$} 
\Comment{Fetch the active domain and service calls for each type}
\State $\curadom{\dt} \gets 
\begin{array}{@{}l@{}l@{}}
&\set{\cv{d}
  \mid \cv{d} \in \tdom{\dt} \cap \cdom{\dfset}}\\
\cup 
% &\left\{\cv{d}
%   \mid \cv{d} \in \tdom{\dt} \cap \left(\dbvar_s^{\cv{a}} \cup
%     \dbvar_s^{\cv{b}} \cup \bigcup_{\tup{\cv{n},\cv{spec}_\cv{n}} \in
%     \curag\setminus\set{\tup{\cv{a},\cv{spec}_\cv{a}},\tup{\cv{b},\cv{spec}_\cv{b}}}}\adom{\state.\db(\cv{n})}\right)
% \right\}
&\left\{\cv{d}
  \mid \cv{d} \in \tdom{\dt} \cap \adom{\state}
\right\}\\
\cup &
\set{\scname{f}(\vec{\cv{o}}) \mid \scname{f}(\vec{\cv{o}}) \in \calls{\dbvar_s^{\cv{a}} \cup
    \dbvar_s^{\cv{b}} } \text{ and } \typed{\scname{f}} =
  \tup{\scname{f}/n,\dfset^{in},\df^{out}} \in \widehat{\scset} \text{
  with } \df^{out} = \tup{\dt,\true}} 
\end{array}$
\EndFor
\State \textbf{pick} 
$\mathfrak{H} \in 
\left\{\tup{\P_1,\ldots,\P_n,\H_1,\ldots,\H_m} \left|
\begin{array}{@{}l@{}}
\P_i \text{ is a }
T_u^i\text{-equality commitment on } \curadom{T_u^i} \text{ for } i
\in \set{1,\ldots,n},\\
\H_j \text{ is a }
T_o^j\text{-densely ordered commitment on } \curadom{T_o^j} \text{ for } j
\in \set{1,\ldots,m}
\end{array}
\right.
\right\}
$
\State $ \sigma \gets \left\{\scname{f}(\vec{\cv{o}}) \mapsto \cv{d}
  \mid \scname{f}(\vec{\cv{o}}) \in \callsvar \text{ and }
  \assignres{\mathfrak{H}}{\state}{\scname{f}(\vec{\cv{o}})}{\ddom{\dtset}} = \cv{d}\right\}$
\State $\dblt \gets \emptyset$ \Comment{Recalculate the
  $\ex{lessThan}$ relations by considering the current active domains
  and the picked commitments}
\ForAll{$i \in \set{1,\ldots,m}$, with $\H_i = \tup{\P'_i,\ex{pos}_i}$}
\ForAll{$\cv{d_1},\cv{d_2} \in \P'_i\sigma$}
\If{$\ex{pos}_i(\cell{\cv{d_1}}{\P'_i\sigma}) <_\mathbb{N} \ex{pos}_i(\cell{\cv{d_2}}{\P'_i\sigma})$}
\State $\dblt \gets \dblt \cup
\set{\lessthan{\dt_o^i}(\cv{d_1},\cv{d_2})}$
\label{alg:lt-update}
\EndIf
\EndFor
\EndFor
\State $\dbvar^{\cv{a}}_{\mathit{cand}} \gets \dbvar_s^{\cv{a}}\sigma, \dbvar^{\cv{b}}_{\mathit{cand}}
\gets \dbvar_s^{\cv{b}}\sigma$
\Comment{Obtain new candidate DBs by substituting service calls with
  results}
\If{$\dbvar^{\cv{a}}_{\mathit{cand}}$ satisfies
  $\flatten(\widehat{\constr}^{\cv{a}})$}
$\dbvar_{\cv{a}} \gets \dbvar^{\cv{a}}_{\mathit{cand}}$
\Comment{Update $\cv{a}$'s DB}
\Else\ 
$\dbvar_{\cv{a}} \gets \state.\db(\cv{a})$
\Comment{Rollback $\cv{a}$'s DB}
\EndIf
\If{$\dbvar^{\cv{b}}_{\mathit{cand}}$ satisfies
  $\flatten(\widehat{\constr}^{\cv{b}})$}
$\dbvar_{\cv{b}} \gets \dbvar^{\cv{b}}_{\mathit{cand}}$
\Comment{Update $\cv{b}$'s DB}
\Else\ 
$\dbvar_{\cv{b}} \gets \state.\db(\cv{b})$
\Comment{Rollback $\cv{b}$'s DB}
\EndIf
\State \textbf{pick} fresh state $\state'$ \Comment{Create new state}
\State $\newag \gets \emptyset$ \Comment{Determine the (possibly
  changed)   set of active agents and their specs}
\If{$\cv{a} = \cv{\inst}$}
 $\newag \gets \set{\tup{\cv{n},\cv{spec}_\cv{n}} \mid \hasSpec(\cv{n},\cv{spec}_\cv{n}) \in \dbvar_{\cv{a}}}$
\ElsIf{$\cv{b} = \cv{\inst}$}
 $\newag \gets \set{\tup{\cv{n},\cv{spec}_\cv{n}} \mid \hasSpec(\cv{n},\cv{spec}_\cv{n}) \in \dbvar_{\cv{b}}}$
\Else\ $\newag \gets \curag$
\Comment{No change if $\inst$ is not involved in the interaction or
  must reject the update}
\EndIf
\ForAll{$\tup{\cv{n},\cv{spec}_\cv{n}} \in \newag$} \Comment{Do the
  update for each active agent}
\If{$\cv{n} = \cv{a}$}
  $\state'.\db(\cv{n}) \gets \dbvar_{\cv{a}}$ 
  \Comment{Case of sender agent} 
\ElsIf{$\cv{n} = \cv{b}$} 
  $\state'.\db(\cv{n}) \gets \dbvar_{\cv{b}}$
\Comment{Case of target agent}
\ElsIf{$\cv{n} \not\in \curag$} 
\Comment{Case of newly created agent}
\State $\state'.\db(\cv{n}) \gets
\idb^{\cv{spec}_\cv{n}} \cup \set{\myname(\cv{n})}$
\Comment{$\cv{n}$'s initial DB gets the initial data fixed by its
  specification, plus its name}
\Else\ $\state'.\db(\cv{n}) \gets \state.\db(\cv{n})$
\Comment{Default case: persisting agent not affected by the interaction}
\EndIf
\EndFor
\State $\dbltplus \gets \set{\lessthan{\dt_o}(\cv{d_1},\cv{d_2}) \mid
  \lessthan{\dt_o}(\cv{d_1},\cv{d_2}) \in \dblt \text{ and }
  \cv{d_1},\cv{d_2} \in \adom{\state} \cup \adom{\state'}}$
\Comment{Filter $\ex{lessThan}$}
\label{alg:lt-filter}
\State $s'.\db(<) \gets \dbltplus$
\Comment{Keep the explicit $\mathit{lessThan}$ relation only for
  persisting objects}
\If{ $\exists \state'' \in \Sigma$ s.t.~$\state''.\db(\inst) =
  \state'.\db(\inst)$ and for each $\tup{\cv{n},\_} \in \curag$,
  $\state''.\db(\cv{n}) = \state'.\db(\cv{n})$ }
\State $\trans \gets \trans \cup \tup{\state,\state''}$
\Comment{State already exists: connect $\state$ to that state}
\Else
\ $\Sigma \gets \Sigma \cup \set{\state'}$, 
$\trans \gets \trans \cup \tup{\state,\state'}$
\Comment{Add and connect new state}
\EndIf
\EndIf
\EndWhile
\EndProcedure

\end{algorithmic}
\vspace{-.3cm}

\caption{Procedure for constructing a transition system that is
  equivalent to that of
   \textsc{build-fb-ts-shallow}, but incorporates the ordering
   relations as special database facts }\label{alg:ts-order}

\end{figure*}

%%% Local Variables:
%%% mode: latex
%%% TeX-master: "main"
%%% save-place: t
%%% End:

\newcommand{\used}{\mathit{UsedObj}}
\newcommand{\passive}{\mathit{PassiveObj}}

\begin{figure*}[!t]
\begin{algorithmic}[1]
\small
\Procedure{build-abstract-ts}{$\widehat{\sys}$}%\Comment{The g.c.d. of a and b}
\\\textbf{input:} Shallow-typed, RMAS $\widehat{X}  = \tup{\dtset, \widehat{\dtset},
  \cdom{\dfset},\widehat{\scset}, \widehat{\evset}}$, with $\dtset =
\set{\dt_u^1,\ldots,\dt_u^n} \cup \set{\dt_o^1,\ldots,\dt_o^m}$\\
\textbf{output:} transition system $\Upsilon_\X=\tup{\ddom{\dt},\Sigma,
  \state_0,\trans}$
\State $\idblt \gets \emptyset$ \Comment{Initial DB incorporating the
  domain-specific $<$ relations for data objects in $\cdom{\dfset}$}
\ForAll{$i \in \set{1,\ldots,m}$}
\ForAll{$\cv{d_1},\cv{d_2} \in \cdom{\dfset} \cap \ddom{\dt_o^m}$}
\If{$\cv{d_1} <_{\dt_o^m}\cv{d_2}$} 
$\idblt \gets \idblt \cup \set{\lessthan{\dt_o^m}(\cv{d_1},\cv{d_2})}$
\EndIf
\EndFor
\EndFor
\State $AS_{0} \gets \set{\tup{\cv{n},\cv{spec}_\cv{n}} \mid
  \hasSpec(\cv{n},\cv{spec}_\cv{n}) \in \idb^\inst}$ \Comment{Initial
  agents with their specifications}
\ForAll{$\tup{\cv{n},\cv{spec}_\cv{n}} \in AS_{0}$}
 $\state_0.\db(\cv{n}) \gets \idb^{\cv{spec}_\cv{n}}$
\Comment{Specify the initial state by extracting the initial DBs from
  the agent specs}
\EndFor
\State $\state_0.\db(<) \gets \idblt$ \Comment{Insert the special
  less-than DB} 
\State $\Sigma \gets \{\state_0\}$, $\trans \gets \emptyset$
\State $\used \gets \cdom{\dfset}$ \Comment{Initialization of the
  container of used data objects}
\While{true}
\State \textbf{pick} $\state \in \Sigma$
\Comment{Nondeterministically pick a state}
\State $\curag \gets \set{\tup{\cv{n},\cv{spec}_\cv{n}} \mid
  \hasSpec(\cv{a},\cv{spec}_\cv{n}) \in \state.\db(\inst)}$
\Comment{Get currently active agents with their specifications}
\State \textbf{pick} $\tup{\cv{a},\cv{spec}_\cv{a}} \in \curag$\Comment{Nondeterministically
  pick an active agent $\cv{a}$, elected as ``sender''}
\State $\enaset \gets
\Call{get-msgs}{\flatten(\widehat{\cspec}^{\cv{spec}_\cv{a}}),s.\db(\cv{a}) \cup
  \state.\db(<),\curag}$ \Comment{Get the enabled messages with target agents}
\If{$\enaset \neq \emptyset$}
\State \textbf{pick} $\tup{M(\vec{\cv{o}}),\cv{b}} \in
\enaset$, with $\tup{\cv{b},\cv{spec}_\cv{b}} \in \curag$
\Comment{Pick a message+target agent and trigger message
  exchange and reactions}
\State $\aset_\cv{a} \gets \emptyset$, $\aset_\cv{b} \gets \emptyset$
\Comment{Get the actions with actual parameters to be applied by
  $\cv{a}$ and $\cv{b}$}
\ForAll{matching on-send rules $``\onev~ M(\vec{x})~\toa~ t~\ifc~ Q(t,\vec{x}) ~\thendo~
  \alpha(t,\vec{x})"$ in $\flatten(\widehat{\urule}^{\cv{spec}_\cv{a}})$} 
\If{$\ans{Q(\cv{b},\vec{\cv{o}}},\state.\db(\cv{a})\cup \state.\db(<))$} 
$\aset_\cv{a} \gets \aset_\cv{a} \cup \alpha(\cv{b},\vec{\cv{o}})$
\EndIf
\EndFor
\ForAll{matching on-receive rules $``\onev~ M(\vec{x})~\froma~ s~\ifc~ Q(s,\vec{x}) ~\thendo~
  \alpha(s,\vec{x})"$  in $\flatten(\widehat{\urule}^{\cv{spec}_\cv{b}})$} 
\If{$\ans{Q(\cv{a},\vec{\cv{o}}},\state.\db(\cv{b}) \cup \state.\db(<))$} 
$\aset_\cv{b} \gets \aset_\cv{b} \cup \alpha(\cv{a},\vec{\cv{o}})$
\EndIf
\EndFor
\State $\tup{\torem^\cv{a},\toaddsk^\cv{a}} \gets
\Call{get-facts}{\flatten(\widehat{\sys}),s.\db(\cv{a})\cup
  \state.\db(<),\aset_{\cv{a}}}$
\State
 $\tup{\torem^\cv{b},\toaddsk^\cv{b}} \gets
\Call{get-facts}{\flatten(\widehat{\sys}),s.\db(\cv{b})\cup \state.\db(<),\aset_{\cv{b}}}$
\State $\dbvar_s^{\cv{a}} \gets (s.\db(\cv{a}) \setminus
\torem^{\cv{a}}) \cup \toaddsk^{\cv{a}}$ \Comment{Calculate new
  $\cv{a}$'s DB, still with service calls to be issued}
\State $\dbvar_s^{\cv{b}} \gets (s.\db(\cv{b}) \setminus
\torem^{\cv{b}}) \cup \toaddsk^{\cv{b}}$ \Comment{Calculate new
  $\cv{b}$'s DB, still with service calls to be issued}
\ForAll{data type $\dt \in \dtset$} 
\Comment{Fetch the active domain and service calls for each type}
\State $\curadom{\dt} \gets 
\begin{array}{@{}l@{}l@{}}
&\set{\cv{d}
  \mid \cv{d} \in \tdom{\dt} \cap \cdom{\dfset}}\\
\cup 
% &\left\{\cv{d}
%   \mid \cv{d} \in \tdom{\dt} \cap \left(\dbvar_s^{\cv{a}} \cup
%     \dbvar_s^{\cv{b}} \cup \bigcup_{\tup{\cv{n},\cv{spec}_\cv{n}} \in
%     \curag\setminus\set{\tup{\cv{a},\cv{spec}_\cv{a}},\tup{\cv{b},\cv{spec}_\cv{b}}}}\adom{\state.\db(\cv{n})}\right)
% \right\}
&\left\{\cv{d}
  \mid \cv{d} \in \tdom{\dt} \cap \adom{\state}
\right\}\\
\cup &
\set{\scname{f}(\vec{\cv{o}}) \mid \scname{f}(\vec{\cv{o}}) \in \calls{\dbvar_s^{\cv{a}} \cup
    \dbvar_s^{\cv{b}} } \text{ and } \typed{\scname{f}} =
  \tup{\scname{f}/n,\dfset^{in},\df^{out}} \in \widehat{\scset} \text{
  with } \df^{out} = \tup{\dt,\true}} 
\end{array}$
\EndFor
\State $\passive \gets \used \setminus \adom{\state} $
\Comment{Calculate passive objects, i.e., data objects used in the
  past but not active now}
\State \textbf{pick} 
$\mathfrak{H} \in 
\left\{\tup{\P_1,\ldots,\P_n,\H_1,\ldots,\H_m} \left|
\begin{array}{@{}l@{}}
\P_i \text{ is a }
T_u^i\text{-equality commitment on } \curadom{T_u^i} \text{ for } i
\in \set{1,\ldots,n},\\
\H_j \text{ is a }
T_o^j\text{-densely ordered commitment on } \curadom{T_o^j} \text{ for } j
\in \set{1,\ldots,m}
\end{array}
\right.
\right\}$
\State $\Delta \gets \ddom{\dtset}$ \Comment{By default, service calls
are substitued with data objects arbitrarily taken from $\ddom{\dtset}$}
\If{$\left|\bigcup_{\P \in \set{\P_1,\ldots,\P_n,\P'_1,\ldots,\P'_m}}
    \set{ec \in \P \mid \text{there is no } \cv{d} \in ec} \right|
  \leq \left| \passive \right|$} \Comment{Sufficiently many passive objects}
\State $\Delta \gets \passive$ \Comment{Pick the fresh results by
  recycling objects in $\passive$ }
\EndIf
\State $ \sigma \gets \left\{\scname{f}(\vec{\cv{o}}) \mapsto \cv{d}
  \mid \scname{f}(\vec{\cv{o}}) \in \callsvar \text{ and }
  \assignres{\mathfrak{H}}{\state}{\scname{f}(\vec{\cv{o}})}{\Delta} =
  \cv{d}\right\}$
\Comment{Get fresh or recycled values}
\State $\dblt \gets \emptyset$ \Comment{Recalculate the
  $\ex{lessThan}$ relations by considering the current active domains
  and the picked commitments}
\ForAll{$i \in \set{1,\ldots,m}$, with $\H_i = \tup{\P'_i,\ex{pos}_i}$}
\ForAll{$\cv{d_1},\cv{d_2} \in \P'_i\sigma$}
\If{$\ex{pos}_i(\cell{\cv{d_1}}{\P'_i\sigma}) <_\mathbb{N} \ex{pos}_i(\cell{\cv{d_2}}{\P'_i\sigma})$}
\State $\dblt \gets \dblt \cup
\set{\lessthan{\dt_o^i}(\cv{d_1},\cv{d_2})}$
\EndIf
\EndFor
\EndFor
\State $\dbvar^{\cv{a}}_{\mathit{cand}} \gets \dbvar_s^{\cv{a}}\sigma, \dbvar^{\cv{b}}_{\mathit{cand}}
\gets \dbvar_s^{\cv{b}}\sigma$
\Comment{Obtain new candidate DBs by substituting service calls with
  results}
\If{$\dbvar^{\cv{a}}_{\mathit{cand}}$ satisfies
  $\flatten(\widehat{\constr}^{\cv{a}})$}
$\dbvar_{\cv{a}} \gets \dbvar^{\cv{a}}_{\mathit{cand}}$
\Comment{Update $\cv{a}$'s DB}
\Else\ 
$\dbvar_{\cv{a}} \gets \state.\db(\cv{a})$
\Comment{Rollback $\cv{a}$'s DB}
\EndIf
\If{$\dbvar^{\cv{b}}_{\mathit{cand}}$ satisfies
  $\flatten(\widehat{\constr}^{\cv{b}})$}
$\dbvar_{\cv{b}} \gets \dbvar^{\cv{b}}_{\mathit{cand}}$
\Comment{Update $\cv{b}$'s DB}
\Else\ 
$\dbvar_{\cv{b}} \gets \state.\db(\cv{b})$
\Comment{Rollback $\cv{b}$'s DB}
\EndIf
\State \textbf{pick} fresh state $\state'$ \Comment{Create new state}
\State $\newag \gets \emptyset$ \Comment{Determine the (possibly
  changed)   set of active agents and their specs}
\If{$\cv{a} = \cv{\inst}$}
 $\newag \gets \set{\tup{\cv{n},\cv{spec}_\cv{n}} \mid \hasSpec(\cv{n},\cv{spec}_\cv{n}) \in \dbvar_{\cv{a}}}$
\ElsIf{$\cv{b} = \cv{\inst}$}
 $\newag \gets \set{\tup{\cv{n},\cv{spec}_\cv{n}} \mid \hasSpec(\cv{n},\cv{spec}_\cv{n}) \in \dbvar_{\cv{b}}}$
\Else\ $\newag \gets \curag$
\Comment{No change if $\inst$ is not involved in the interaction or
  must reject the update}
\EndIf
\ForAll{$\tup{\cv{n},\cv{spec}_\cv{n}} \in \newag$} \Comment{Do the
  update for each active agent}
\If{$\cv{n} = \cv{a}$}
  $\state'.\db(\cv{n}) \gets \dbvar_{\cv{a}}$ 
  \Comment{Case of sender agent} 
\ElsIf{$\cv{n} = \cv{b}$} 
  $\state'.\db(\cv{n}) \gets \dbvar_{\cv{b}}$
\Comment{Case of target agent}
\ElsIf{$\cv{n} \not\in \curag$} 
\Comment{Case of newly created agent}
\State $\state'.\db(\cv{n}) \gets
\idb^{\cv{spec}_\cv{n}} \cup \set{\myname(\cv{n})}$
\Comment{$\cv{n}$'s initial DB gets the initial data fixed by its
  specification, plus its name}
\Else\ $\state'.\db(\cv{n}) \gets \state.\db(\cv{n})$
\Comment{Default case: persisting agent not affected by the interaction}
\EndIf
\EndFor
\State $\dbltplus \gets \set{\lessthan{\dt_o}(\cv{d_1},\cv{d_2}) \mid
  \lessthan{\dt_o}(\cv{d_1},\cv{d_2}) \in \dblt \text{ and }
  \cv{d_1},\cv{d_2} \in \adom{\state} \cup \adom{\state'}}$
\Comment{Filter $\ex{lessThan}$}
\State $s'.\db(<) \gets \dbltplus$
\Comment{Keep the explicit $\mathit{lessThan}$ relation only for
  persisting objects}
\If{ $\exists \state'' \in \Sigma$ s.t.~$\state''.\db(\inst) =
  \state'.\db(\inst)$ and for each $\tup{\cv{n},\_} \in \curag$,
  $\state''.\db(\cv{n}) = \state'.\db(\cv{n})$ }
\State $\trans \gets \trans \cup \tup{\state,\state''}$
\Comment{State already exists: connect $\state$ to that state}
\Else
\ $\Sigma \gets \Sigma \cup \set{\state'}$, 
$\trans \gets \trans \cup \tup{\state,\state'}$
\Comment{Add and connect new state}
\EndIf
\EndIf
\EndWhile
\EndProcedure

\end{algorithmic}
\vspace{-.3cm}

\caption{Procedure for constructing a sound and complete abstraction
  of the transition system constructed with the 
   \textsc{build-fb-ts-shallow-flat} procedure, by recycling
   non-persisting data objects}\label{alg:ts-rcycl}

\end{figure*}

%%% Local Variables:
%%% mode: latex
%%% TeX-master: "main"
%%% save-place: t
%%% End:

% By combining all these results, we obtain: for every
% closed $\mulpd$ formula $\Phi$, $\ts_{\sys} \models \Phi$ if and only
% if $\fts_{\widehat{\sys}} \models \Phi$.

%%% Local Variables:
%%% mode: latex
%%% TeX-master: "main"
%%% save-place: t
%%% End:

\section{Conclusion}

RMASs constitute a very rich modeling framework for data-aware multiagent
systems. The presence of concrete data types and their facets greatly empowers
its modeling capabilities, making it, e.g., apt to capture mutual exclusion
protocols, asynchronous interactions with bounded queues, and price-based
protocols. Our key result, namely that densely-order, state-bounded RMASs are
verifiable with standard model checking techniques, paves the way towards
concrete verification algorithms for this class of systems
\cite{LoQR09,CCDG*14}. In this respect, a major obstacle is the exponentiality
in the data slots that can be changed over time, a source of complexity that is
inherent in all data-aware dynamic systems \cite{DeSV07}. We intend to attack
this by proposing data modularization techniques to decompose the system into
smaller components.

 From a foundational perspective, our work presents connections to
 \cite{Bela14}, which extends the framework in \cite{BeLP12} with types so as
 to model and verify auctions.  The two settings are incomparable w.r.t.\ both
 the framework and the verification logic, and it would be interesting to
 study cross-transfer of results between the two settings.

% A natural development of this work concerns the extension of our
% results to generalizations of data types such as those studied in
% \cite{}, for which decidability results have been recently assessed in
% the Situation Calculus \cite{} (without considering dense linear
% orders, though). Furthermore, dense orders naturally lend RMASs to
% capture multi-party auctions, \cite{}.

\bibliographystyle{aaai}
%\bibliography{string-medium,local}
\bibliography{main-bib}

% \clearpage
% \appendix
% \section{Transition System Construction}
% \include{ts}

% \section{ Acknowledgments}
% AAAI is especially grateful to Peter Patel Schneider for his work in
% implementing the aaai.sty file, liberally using the ideas of other style
% hackers, including Barbara Beeton. We also acknowledge with thanks the work
% of George Ferguson for his guide to using the style and BibTeX files ---
% which has been incorporated into this document --- and Hans Guesgen, who
% provided several timely modifications, as well as the many others who have,
% from time to time, sent in suggestions on improvements to the AAAI style.

\end{document}

\endinput

%%% Local Variables:
%%% mode: latex
%%% TeX-master: t
%%% save-place: t
%%% End: